\documentclass[final,12pt]{colt2021} % Include author names
\usepackage{times}
\usepackage{dsfont}
\usepackage{enumitem}
\newcommand{\C}{G}
\usepackage[english]{babel}
\usepackage{amsmath, amssymb}
\usepackage{graphicx}
\usepackage{latexsym}
\usepackage{graphicx}
\usepackage{cases}
\usepackage[mathscr]{euscript}
\usepackage{xcolor}
\usepackage{colortbl}
\usepackage{thmtools}
\usepackage{thm-restate}
\usepackage{mathtools}
\usepackage{caption}

\usepackage{pifont}% http://ctan.org/pkg/pifont

\let\vec\undefined
\usepackage{MnSymbol}
\DeclareMathAlphabet\mathbb{U}{msb}{m}{n}
\usepackage{xpatch}

\definecolor{Gray}{gray}{0.85}
\newcolumntype{g}{>{\columncolor{Gray}}c}

\hypersetup{
  colorlinks   = true, %Colours links instead of ugly boxes
  urlcolor     = blue, %Colour for external hyperlinks
  linkcolor    = blue, %Colour of internal links
  citecolor   = blue %Colour of citations
}

\def\Rset{\mathbb{R}}
\def\Nset{\mathbb{N}}

\DeclareMathOperator*{\E}{\mathbb{E}}

\newcommand{\nrm}[1]{{\left\vert\kern-0.25ex\left\vert\kern-0.25ex\left\vert #1 
    \right\vert\kern-0.25ex\right\vert\kern-0.25ex\right\vert}}

\DeclarePairedDelimiter{\curl}{\{}{\}}
\DeclarePairedDelimiter{\paren}{(}{)}

\newcommand{\cC}{\mathcal{C}}

\newcommand{\cH}{\mathcal{H}}

\newcommand{\cR}{\mathcal{R}}

\newcommand{\cX}{\mathcal{X}}
\newcommand{\cY}{\mathcal{Y}}

\newcommand{\sH}{{\mathscr H}}

\newcommand{\sP}{{\mathscr P}}

\newcommand{\sX}{{\mathscr X}}
\newcommand{\sY}{{\mathscr Y}}

\newcommand{\bu}{{\mathbf u}}

\newcommand{\bw}{{\mathbf w}}
\newcommand{\bx}{{\mathbf x}}
\newcommand{\bz}{{\mathbf z}}

\newcommand{\ignore}[1]{}

\title[A Finer Calibration Analysis for Adversarial Robustness]
{A Finer Calibration Analysis for Adversarial Robustness}

\coltauthor{%
 \Name{Pranjal Awasthi} \Email{pranjal.awasthi@rutgers.edu}\\
 \addr Google Research and Rutgers University, New York
 \AND
 \Name{Anqi Mao} \Email{aqmao@cims.nyu.edu}\\
 \addr Courant Institute of Mathematical Sciences, New York%
 \AND
 \Name{Mehryar Mohri} \Email{mohri@google.com}\\
 \addr Google Research and Courant Institute of Mathematical Sciences, New York%
 \AND
 \Name{Yutao Zhong} \Email{yutao@cims.nyu.edu}\\
 \addr Courant Institute of Mathematical Sciences, New York%
}
\usepackage[toc,page,header]{appendix}

\begin{document}

\maketitle

\begin{abstract}%
We present a more general analysis of \emph{$\sH$-calibration} for
adversarially robust classification. By adopting a finer definition of calibration, we can cover settings beyond the restricted hypothesis sets studied in previous work.  In particular, our results hold for
most common hypothesis sets used in machine learning. We both fix
some previous calibration results \citep{pmlr-v125-bao20a} and
generalize others \citep{PNAMY21}. Moreover, our calibration results,
combined with the previous study of consistency by \citet{PNAMY21},
also lead to more general \emph{$\sH$-consistency} results covering common hypothesis sets.

\end{abstract}


\section{Introduction}
\label{sec:introduction}

Rich learning models trained on large datasets often achieve a high
accuracy in a variety of applications
\citep{SutskeverVinyalsLe2014,KrizhevskySutskeverHinton2012}. However,
such complex models have been shown to be susceptible to imperceptible perturbations
\citep{szegedy2013intriguing}: an unnoticeable perturbation can, for example,
result in a dog being classified as an electronics device, which could lead
to dramatic consequences in practice in many applications.

This has motivated the introduction and analysis of the notion of
\emph{adversarial loss}, which requires a predictor not only to
correctly classify an input point $\bx$ but also to maintain the same
classification for all points at a small $\ell_p$ distance of $\bx$
\citep{goodfellow2014explaining, madry2017towards,
tsipras2018robustness, carlini2017towards}.

The problem of designing effective learning algorithms 
with theoretical guarantees for the \emph{adversarial loss} has been the
topic of a number of recent studies \citep{pmlr-v125-bao20a, PNAMY21}. 
In particular, these authors have initiated a theoretical analysis
of the \emph{$\sH$-calibration} and \emph{$\sH$-consistency} of surrogate losses
for the \emph{adversarial $0/1$ loss}.

\citet{pmlr-v125-bao20a} analyzed \emph{$\sH$-calibration} for adversarially
robust classification in the special case where $\sH$ is the family of
linear models. However, several comments are due regarding that work.
First, the definition of calibration adopted by the authors does not
coincide with the standard definition \citep{steinwart2007compare} in
the case of the linear models they study, although it does match that
definition in the case of the family of all measurable functions
\citep[Section~4.1]{steinwart2007compare}: the minimal inner risk in
the definition should be defined for a fixed $\bx$ and the infimum
should be over $f$, instead of an infimum over both $f$ and $\bx$.
Second, and this is crucial, \emph{$\sH$-calibration}, in general, does
not imply \emph{$\sH$-consistency}, unless a property such as
\emph{$\sP$-minimizability} holds
\citep[Theorem~2.8]{steinwart2007compare}. \emph{$\sP$-minimizability} holds for standard
binary classification and the family of all measurable functions
\citep[Theorem~3.2]{steinwart2007compare}. However, it does not hold,
in general, for adversarially robust classification and a specific
hypothesis set $\sH$. As a result, the claim made by the authors that
the calibrated surrogates they propose are \emph{$\sH$-consistent} is
incorrect, as shown by \citet{PNAMY21}.
Third, the authors analyze \emph{$\sH$-calibration} with respect to the loss
function $\phi_{\gamma}\colon \bx \mapsto \mathds{1}_{yf(\bx)\leq
\gamma}$ in the case where $\sH\supset [-1,1]$ is the general family
of functions. However, $\phi_{\gamma}$ only coincides with the
\emph{adversarial $0/1$ loss $\ell_{\gamma}$} in
Equation~\eqref{eq:supinf01} in the special case where $\sH$ is the
family of linear models \citep[Proposition~1]{pmlr-v125-bao20a}.

\citet{PNAMY21} also recently studied the \emph{$\sH$-calibration} and
\emph{$\sH$-consistency} of adversarial surrogate losses. They pointed out
the issues just mentioned about the study of \citet{pmlr-v125-bao20a}
and considered more general hypothesis sets, such as generalized
linear models, ReLU-based functions, and one-layer ReLU neural networks. They identified natural conditions under which
\emph{$\sH$-calibrated} losses can be \emph{$\sH$-consistent} in the adversarial
scenario. They also derived calibration results under the correct
definition of the minimal inner risk by analyzing the equivalence of
two definitions.
However, with this method of calibration analysis, the calibration
considered by the authors needs to be a uniform calibration
\citep[Definition~2.15]{steinwart2007compare} instead of non-uniform
calibration \citep[Definition~2.7]{steinwart2007compare}. In view of
that, their positive result imposes an extra restriction on the
parameters of the hypothesis sets, which can be removed through the
analysis presented here.
  
\noindent \textbf{Our Contributions.}  Building on previous work by
\citet{PNAMY21}, we present a more general analysis of
\emph{$\sH$-calibration} for adversarially robust classification for more
general hypothesis sets. For example, our Theorem~\ref{Thm:calibration_margin_convex}, Theorem~\ref{Thm:calibration_sup_convex} and Theorem~\ref{Thm:sup_rhomargin_calibrate_general} apply to most common
hypothesis sets. Furthermore, for the specific hypothesis sets
considered in previous work, our results either fix existing
calibration results \citep{pmlr-v125-bao20a} or generalize them
\citep{PNAMY21}. More precisely, our Theorem~\ref{Thm:calibration_positive_linear} is a correction to the main
positive result, Theorem~11 in \citep{pmlr-v125-bao20a}, where we
prove the theorem under the correct calibration definition.
Moreover, our Theorem~\ref{Thm:quasiconcave_calibrate_general} extends
the results for linear models to generalized linear models. Our
Corollary~\ref{Corollary:calibration_margin_convex}, Theorem~\ref{Thm:linear_sup_convex}, Theorem~\ref{Thm:calibration_sup_convex} and Corollary~\ref{Corollary:calibration_sup_convex} are stronger versions
of the negative calibration results Theorem~10, Corollary~11,
Theorem~12 and Corollary~13 in \citep{PNAMY21}, since the calibration
considered in \citep{PNAMY21} is uniform calibration \citep[Definition~2.15]{steinwart2007compare}, which is stronger than non-uniform
calibration \citep[Definition~2.7]{steinwart2007compare} considered in
our paper. Our Theorem~\ref{theorem:rho_margin_calibtartion} and
Corollary~\ref{Corollary:sup_rhomargin_calibrate_general} are
generalizations of the positive calibration results of
\citet{PNAMY21}, since our results hold without the unboundedness
assumptions for parameters of the hypothesis sets.

\section{Preliminaries}
\label{sec:preliminaries}

We adopt much of the notation used in \citep{PNAMY21}.
We will denote vectors as lowercase bold letters (e.g.\ $\bx$). The
$d$-dimensional $l_2$-ball with radius $r$ is denoted by
$B_2^d(r)\colon=\curl*{\bz\in\mathbb{R}^d\mid\|\bz\|_2\leq r }$. We
denote by $\sX$ the set of all possible examples. $\sX$ is also
sometimes referred to as the input space. The set of all possible
labels is denoted by $\sY$. We will limit ourselves to the case of
binary classification where $\sY=\{-1,1\}$. Let $\sH$ be a family of
functions from $\Rset^d$ to $\Rset$. Given a fixed but unknown
distribution $\sP$ over $\sX\times\sY$, the binary classification
learning problem is then formulated as follows. The learner seeks
to select a predictor $f\in \sH$ with small 
\emph{generalization error} with respect to the distribution $\sP$. 
The \emph{generalization error} of a
classifier $f \in \sH$ is defined by $\cR_{\ell_0}(f) = \E_{(\bx, y)
  \sim \sP}[\ell_0(f,\bx,y)]$, where $\ell_0(f,\bx,y)=\mathds{1}_{y
  f(\bx) \leq 0}$ is the standard $0/1$ loss. More generally, the
\emph{$\ell$-risk} of a classifier $f$ for a surrogate loss
$\ell(f,\bx,y)$ is defined by
\begin{align}
\label{eq:surrogate-risk}
    \cR_{\ell}(f) = \E_{(\bx, y) \sim \sP}[\ell(f, \bx, y)].
\end{align}
Moreover, the \emph{minimal ($\ell$,$\sH$)-risk}, which is also called
the \emph{Bayes ($\ell$,$\sH$)-risk}, is defined by
$\cR_{\ell,\sH}^*=\inf_{f\in\sH}\cR_{\ell}(f)$. In the
standard classification setting, the goal of a consistency analysis is to
determine whether the minimization of a surrogate loss $\ell$ can lead to
that of the binary loss generalization error. Similarly, in adversarially robust classification, the goal of a consistency analysis is to determine
if the minimization of a surrogate loss $\ell$ yields that of the \emph{adversarial generalization error} defined by $\cR_{\ell_{\gamma}}(f)=\E_{(\bx, y) \sim \sP}[\ell_{\gamma}(f,\bx,y)]$, where 
\begin{equation}
\label{eq:AdvLoss}
\ell_{\gamma}(f,\bx,y)\colon = \sup_{\bx'\colon \|\bx-\bx'\|\leq \gamma}\mathds{1}_{y f(\bx') \leq 0}
\end{equation}
is the \emph{adversarial $0/1$ loss}. This motivates the 
definition of
\emph{$\sH$-consistency} (or simply \emph{consistency}) stated below.
\begin{definition}[$\sH$-Consistency]
  Given a hypothesis set $\sH$, we say that a loss function $\ell_1$
  is \emph{$\sH$-consistent} with respect to loss function $\ell_2$,
  if the following holds:
\begin{align}
\label{eq:H-consistency}
    \cR_{\ell_1}(f_n)-\cR_{\ell_1,\sH}^* \xrightarrow{n \rightarrow +\infty} 0 \implies \cR_{\ell_2}(f_n)-\cR_{\ell_2,\sH}^* \xrightarrow{n \rightarrow +\infty} 0,
\end{align}
for all probability distributions and sequences of $\{f_n\}_{n\in \Nset}\subset \sH$.
\end{definition}
% In the rest of the paper, the loss $\ell_2$ in the definition above
% will correspond to the $0/1$ loss or the adversarial $0/1$ loss
% depending on the context, $\ell_1$ to a surrogate loss for
% $\ell_2$. 
For a distribution $\sP$ over $\sX \times \sY$ with random
variables $X$ and $Y$, let $\eta_{\sP} \colon\sX \rightarrow [0,1]$ be a
measurable function such that, for any $\bx \in \sX$, $\eta_{\sP}(\bx) =
\sP(Y = 1 \mid X=\bx)$. By the property of conditional expectation, we
can rewrite \eqref{eq:surrogate-risk} as $\cR_{\ell}(f) =
\E_{X}[\cC_{\ell}(f, \bx, \eta_{\sP}(\bx))]$, where $\cC_{\ell}(f, \bx,
\eta)$ is the \emph{generic conditional $\ell$-risk} (or
\emph{inner $\ell$-risk}) defined as followed:
\begin{align}
\label{eq:conditional-risk}
\forall \bx \in \sX, \forall \eta \in [0, 1], \quad 
\cC_{\ell}(f,\bx,\eta)\colon 
= \eta \ell(f, \bx, +1) + (1 - \eta)\ell(f, \bx, -1).
\end{align}
Moreover, the \emph{minimal inner $\ell$-risk} on $\sH$ is denoted by
$\cC_{\ell,\sH}^{*}(\bx,\eta)\colon=\inf_{f\in\sH}\cC_{\ell}(f,\bx,\eta).$ The notion of \emph{calibration} for the inner risk is often a powerful tool for
the analysis of $\sH$-consistency \citep{steinwart2007compare}.
\begin{definition}[$\sH$-Calibration]\emph{[Definition~2.7 in \citep{steinwart2007compare}]}
\label{def:H-calibration_true}
Given a hypothesis set $\sH$, we say that a loss function $\ell_1$ is
\emph{$\sH$-calibrated} with respect to a loss function $\ell_2$
if, for any $\epsilon>0$, $\eta\in
     [0,1]$, and $\bx \in \sX$, there exists $\delta>0$ such that for all $f\in\sH$ we have
\begin{align}
\label{eq:H-calibration_true}
    \cC_{\ell_1}(f,\bx,\eta)<\cC_{\ell_1,\sH}^{*}(\bx,\eta)+\delta \implies \cC_{\ell_2}(f,\bx,\eta)<\cC_{\ell_2,\sH}^{*}(\bx,\eta)+\epsilon.
\end{align}
\end{definition}
For comparison with previous work, we also introduce the \emph{uniform $\sH$-calibration} in \citep{steinwart2007compare}, which is stronger than Definition~\ref{def:H-calibration_true}.
\begin{definition}[Uniform $\sH$-Calibration]\emph{[Definition~2.15 in \citep{steinwart2007compare}]}
\label{def:H-calibration_real}
Given a hypothesis set $\sH$, we say that a loss function $\ell_1$ is
\emph{uniform $\sH$-calibrated} with respect to a loss function $\ell_2$
if, for any
$\epsilon>0$, there exists $\delta>0$ such that for all $\eta\in
     [0,1]$, $f\in\sH$, $\bx \in \sX$, we have
\begin{align}
\label{eq:H-calibration_real}
    \cC_{\ell_1}(f,\bx,\eta)<\cC_{\ell_1,\sH}^{*}(\bx,\eta)+\delta \implies \cC_{\ell_2}(f,\bx,\eta)<\cC_{\ell_2,\sH}^{*}(\bx,\eta)+\epsilon.\end{align}
\end{definition}
Note that, in the previous work of \citet{PNAMY21}, Definition~\ref{def:H-calibration_real} is adopted, where $\delta$ in \eqref{eq:H-calibration_real} is independent of $\eta$ and $\bx$; the work of \citet{pmlr-v125-bao20a} adopts a similar definition.
%which is also uniform with $\delta$. 
% The work of \citet[Theorem 2.13]{steinwart2007compare} establishes that the \emph{excess risk} of a surrogate loss $\ell_1$ can be upper bounded in terms of the \emph{excess risk} of target loss $\ell_2$ using an increasing function whenever $\ell_1$ is uniformly calibrated with respect to $\ell_2$. Moreover, \citet[Theorem 4.3]{steinwart2007compare} shows that calibration and uniform calibration are equivalent in standard binary classification when considering the family of all measurable functions. However, this is not true for adversarially robust classification. 
In this paper, we will focus on the non-uniform case, that is Definition~\ref{def:H-calibration_true}, where $\delta$ is dependent on $\eta$ and $\bx$. There are two advantages to considering non-uniform calibration:
% instead of uniform one in adversarial robust classification: 
it makes it possible to provide stronger negative results on calibration properties of convex surrogates and, it helps us prove more general positive results that hold for most common hypothesis sets $\sH$. In contrast, positive results for uniform calibration hold for some restricted hypothesis sets \citep{PNAMY21}.

\citet{steinwart2007compare} showed that if $\ell_1$ is $\sH$-calibrated (it suffices to satisfy non-uniform calibration, that is condition~\eqref{eq:H-calibration_true}) with respect to $\ell_2$,  then $\sH$-consistency, that is condition~\eqref{eq:H-consistency}, holds for any probability distribution verifying the 
additional condition of
\emph{$\sP$-minimizability} \citep[Definition~2.4]{steinwart2007compare}. While \emph{$\sP$-minimizability} does not hold in general for adversarially robust classification, \citet{PNAMY21} showed that the uniform $\sH$-calibrated losses are $\sH$-consistent under certain conditions. In fact, it also suffices to satisfy non-uniform calibration, that is condition~\eqref{eq:H-calibration_true} for these results, since their proofs only make use of the weaker non-uniform property.

Next, we introduce the notions of \emph{calibration function} and an important result characterizing $\sH$-calibration from \citep{steinwart2007compare}.
\begin{definition}[Calibration function]
Given a hypothesis set $\sH$, we define the \emph{calibration function} $\delta_{\max}$ for a pair of losses $(\ell_1,\ell_2)$ as follows: for all $\bx\in \sX$, $\eta\in [0,1]$ and $\epsilon>0$,
\begin{align}
\label{eq:def-calibration-function-true}
    &\delta_{\max}(\epsilon,\bx,\eta) =\inf_{f\in\sH} \curl[\Big]{\cC_{\ell_1}(f,\bx,\eta) - \cC^{*}_{\ell_1,\sH}(\bx,\eta) \mid
    \cC_{\ell_2}(f,\bx,\eta) - \cC^{*}_{\ell_2,\sH}(\bx,\eta)\geq\epsilon}\,.
\end{align}
\end{definition}
\begin{proposition}[Lemma~2.9 in \citep{steinwart2007compare}]
\label{prop:calibration_function_positive}
Given a hypothesis set $\sH$, loss $\ell_1$ is $\sH$-calibrated with respect to $\ell_2$ if and only if its calibration function $\delta_{\max}$ satisfies
$\delta_{\max}(\epsilon,\bx,\eta)>0$ for all $\bx\in \sX$, $\eta\in
     [0,1]$ and $\epsilon>0$.
\end{proposition}
For comparison, \citet[Definition 3]{pmlr-v125-bao20a} and \citet[Definition 2]{PNAMY21} consider the \emph{Uniform Calibration function} $\delta(\epsilon)$ and make use of Lemma~2.16 in \citep{steinwart2007compare} to characterize uniform calibration \citep[Proposition 4]{PNAMY21,pmlr-v125-bao20a}. Note $\delta(\epsilon)>0$ implies $\delta_{\max}(\epsilon,\bx,\eta)>0$ for all $\bx\in \sX$, $\eta\in [0,1]$, and as a result uniform calibration implies non-uniform calibration. However, the converse does not hold in general.
%but it is not true on the contrary.

\section{Adversarially Robust Classification}
\label{sec:robust_classification}
In adversarially robust classification, the loss at $(\bx,y)$ is measured in terms of the
worst loss incurred over an adversarial perturbation of $\bx$ within a
ball of a certain radius in a norm. In this work we will consider
perturbations in the $l_2$ norm $\|\cdot\|$. We will denote by
$\gamma$ the maximum magnitude of the allowed perturbations. Given
$\gamma>0$, a data point $(\bx,y)$, a function $f\in \sH$, and a
margin-based loss $\phi\colon\mathbb{R}\rightarrow\mathbb{R}_{+}$, we
define the \emph{adversarial loss} of $f$ at $(\bx,y)$ as
\begin{align}
\label{eq:sup-based-surrogate}
\tilde{\phi}(f,\bx,y)=\sup\limits_{\bx'\colon \|\bx-\bx'\|\leq \gamma}\phi(y f(\bx')).
\end{align}
The above naturally motivates {\em supremum-based} surrogate losses
that are commonly used to optimize the adversarial $0/1$
loss \citep{goodfellow2014explaining, madry2017towards,
  shafahi2019adversarial, wong2020fast}. We say that a surrogate loss
$\tilde{\phi}(f, \bx, y)$ is \emph{supremum-based} if it is of the
form defined in \eqref{eq:sup-based-surrogate}. We say that the
supremum-based surrogate is convex if the function $\phi$ in
\eqref{eq:sup-based-surrogate} is convex. When $\phi$ is
non-increasing, the following equality
holds \citep{YinRamchandranBartlett2019}:
\begin{align}
\label{eq:sup=inf}
\sup\limits_{\bx'\colon \|\bx-\bx'\|\leq \gamma}\phi(y f(\bx')) = \phi\paren*{\inf\limits_{\bx'\colon \|\bx-\bx'\|\leq \gamma}yf(\bx')}.
\end{align}
The adversarial $0/1$ loss defined in
\eqref{eq:AdvLoss} is a special kind of adversarial loss \eqref{eq:sup-based-surrogate}, where $\phi$ is the $0/1$ loss, that is, $\phi(yf(\bx))=\ell_0(f,\bx,y)=\mathds{1}_{yf(\bx)\leq 0}$. Therefore, the adversarial $0/1$ loss has the equivalent form
\begin{align}
\label{eq:supinf01}
    \ell_{\gamma}(f,\bx,y)=\sup\limits_{\bx'\colon \|\bx-\bx'\|\leq \gamma}\mathds{1}_{y f(\bx') \leq 0}=\mathds{1}_{\inf\limits_{\bx'\colon \|\bx-\bx'\|\leq \gamma}yf(\bx')\leq 0}.
\end{align}
This alternative equivalent form of adversarial $0/1$ loss is more advantageous to analyze than \eqref{eq:AdvLoss} and would be adopted in our proofs.
% As with the standard binary classification, we define the \emph{adversarial generalization error} and the \emph{Bayes ($\ell_{\gamma}$,$\sH$)-risk} as
% \[
% \cR_{\ell_{\gamma}}(f)=\E_{(\bx, y) \sim \sP}[\ell_{\gamma}(f,\bx,y)] \quad \text{and} \quad \cR_{\ell_{\gamma},\sH}^*=\inf_{f\in\sH}\cR_{\ell_{\gamma}}(f).
% \]
Without loss of generality, let $\cX=B_2^d(1)$ and $\gamma\in (0,1)$. In this paper, we aim to characterize surrogate losses $\ell_{1}$ satisfying $\sH$-calibration \eqref{eq:H-calibration_true} with $\ell_2 = \ell_{\gamma}$ and for the hypothesis sets $\sH$ which are \emph{regular for adversarial calibration}.
\begin{definition}[Regularity for Adversarial Calibration]
We say that a hypothesis set $\sH$ is \emph{regular for adversarial calibration} 
if there exists a \emph{distinguishing $\bx$} in $\sX$, that is if there exist $f, g \in \sH$ 
such that $\inf_{\| \bx' - \bx \| \leq \gamma } f(\bx')>0$ and 
$\sup_{\| \bx' - \bx \| \leq \gamma }g(\bx')<0$.
\end{definition}
It suffices to study hypothesis sets $\sH$ that are regular for adversarial calibration not only because all common hypothesis sets admit that property, but also because the following result holds. We say that a hypothesis set $\sH$ is \emph{symmetric}, 
if for any $f \in \sH$, $-f$ is also in $\sH$.
\begin{restatable}{theorem}{NatrualCondition}
\label{Thm:natural_condition}
   Let $\sH$ be a symmetric hypothesis set. If $\sH$ is not regular for adversarial calibration, 
   %that is, any $\bx \in \sX$ is not distinguishing.
   then any surrogate loss $\ell$ is $\sH$-calibrated with respect to $\ell_{\gamma}$.
\end{restatable}
\begin{proof}
Since $\sH$ is symmetric, for any $\bx \in \sX$, $f\in \sH$, $\inf_{\| \bx' - \bx \| \leq \gamma } f(\bx')\leq 0 \leq \sup_{\| \bx' - \bx \| \leq \gamma }f(\bx')$. Thus by the definition of inner risk \eqref{eq:conditional-risk} and adversarial 0-1 loss $\ell_{\gamma}$ \eqref{eq:supinf01}, for any $\bx \in \sX$, $f\in \sH$,
\[\cC_{\ell_{\gamma},\sH}(f,\bx,\eta)=\eta \mathds{1}_{\inf\limits_{\bx'\colon \|\bx-\bx'\|\leq \gamma} f(\bx')\leq 0} +(1-\eta) \mathds{1}_{\sup\limits_{\bx'\colon \|\bx-\bx'\|\leq \gamma} f(\bx')\geq 0}=1=\cC^*_{\ell_{\gamma},\sH}(\bx,\eta),\]
which implies any surrogate loss $\ell$ is $\sH$-calibrated with respect to $\ell_{\gamma}$ by \eqref{eq:H-calibration_true}.
\end{proof}
Note all the hypothesis sets considered in the previous work \citep{pmlr-v125-bao20a} and \citep{PNAMY21} are regular for adversarial calibration. For convenience, we adopt the notation in \citep{PNAMY21} to denote these specific hypothesis sets:
\begin{itemize}[itemsep=-1mm]
\item linear models: $\sH_{\mathrm{lin}}=\curl*{\bx\rightarrow \bw \cdot \bx \mid \|\bw\|=1}$, as in \citep{pmlr-v125-bao20a} and \citep{PNAMY21}.

\item generalized linear models: $\sH_{g} = \curl*{\bx\rightarrow
  g(\bw \cdot \bx)+b\mid\|\bw\|= 1, |b|\leq \C}$ where $g$ is a
  non-decreasing function, as in \citep{PNAMY21}; and

\item one-layer ReLU neural networks:
$\sH_{\mathrm{NN}} = \curl*{\bx\rightarrow \sum_{j = 1}^n u_j(\bw_j \cdot \bx)_{+} \mid \|\bu \|_{1}\leq \Lambda, \|\bw_j\|\leq W}$, where $(\cdot)_+ = \max(\cdot,0)$ as in \citep{PNAMY21}; and

\item all measurable functions: $\sH_{\mathrm{all}}$ as in \citep{PNAMY21}.
\end{itemize}
In the special case of $g = (\cdot)_+$, 
we denote the corresponding ReLU-based hypothesis set as $\sH_{\mathrm{relu}} = \curl*{\bx\rightarrow (\bw \cdot \bx)_{+} + b \mid \|\bw\|=1, |b|\leq \C}$ as in \citep{PNAMY21}.

\section{\texorpdfstring{$\sH$}{H}-Calibration Analysis}
\label{sec:calibration}

\subsection{Negative results}
\label{sec:calibration_negative}

In this section, we show that the commonly used convex surrogates and supremum-based convex surrogates are not $\sH$-calibrated with respect to $\ell_{\gamma}$, even under the weaker notion of non-uniform calibration. These results can be viewed as a generalization of those
given by \citet{PNAMY21}.
% aim to study that common losses are not $\sH$-calibrated with respect to $\ell_{\gamma}$.
\subsubsection{Convex losses}
\label{sec:calibration_negative_convex}
We first study convex losses, which are often used 
for standard binary classification problems.
\begin{restatable}{theorem}{CalibrationMarginConvex}
\label{Thm:calibration_margin_convex}
    Assume $\sH$ satisfies there exists a distinguishing $\bx_0\in \sX$ and $f_0\in \sH$ such that $f_0(\bx_0)=0$. If a margin-based loss $\phi\colon\Rset\rightarrow \Rset_{+}$ is convex, then it is not $\sH$-calibrated with respect to $\ell_{\gamma}$.
\end{restatable}
In particular, the assumption holds when $\sH$ is regular for adversarial calibration and contains $0$.
The proof of Theorem~\ref{Thm:calibration_margin_convex} is included in Appendix~\ref{app:calibration_convex}. By Theorem~\ref{Thm:calibration_margin_convex}, we obtain the following corollary, which fixes the main negative result of \citet{pmlr-v125-bao20a} and generalizes negative results of \citet{PNAMY21}. Note $\sH_{\mathrm{lin}}$, $\sH_{\mathrm{NN}}$ and $\sH_{\mathrm{all}}$ all satisfy there exists a distinguishing $\bx_0\in \sX$ and $f_0\in \sH$ such that $f_0(\bx_0)=0$. When $g(-\gamma)+\C>0$ and  $g(-\gamma) -\C<0$, $\sH_g$ also satisfies this assumption.
\begin{corollary} 
\label{Corollary:calibration_margin_convex}
If a margin-based loss $\phi\colon\Rset\rightarrow \Rset_{+}$ is convex, then,
  \begin{enumerate}
      \item $\phi$ is not $\sH_{\mathrm{lin}}$-calibrated with respect to $\ell_{\gamma}$;      
      \item Given a non-decreasing and continuous function $g$ such that $g(-\gamma)+\C>0$ and  $g(\gamma) -\C<0$. Then $\phi$ is not $\sH_g$-calibrated with respect to $\ell_{\gamma}$; Specifically, if $\C>\gamma$, then $\phi$ is not $\sH_{\mathrm{relu}}$-calibrated with respect to $\ell_{\gamma}$;
      \item $\phi$ is not $\sH_{\mathrm{NN}}$-calibrated with respect to $\ell_{\gamma}$;  
      \item $\phi$ is not $\sH_{\mathrm{all}}$-calibrated with respect to $\ell_{\gamma}$.
  \end{enumerate}
\end{corollary}
By using the correct calibration Definition~\ref{def:H-calibration_true}, 1.\ of Corollary~\ref{Corollary:calibration_margin_convex} fixes the main negative result in \citep{pmlr-v125-bao20a}. 

\subsubsection{Supremum-based convex losses}
\label{sec:sup_based_convex}

While it is natural to consider convex surrogates for the $0/1$ loss, convex supremum-based surrogates are widely used in practice for designing algorithms for the adversarial loss 
\citep{madry2017towards, shafahi2019adversarial, wong2020fast}. We next present negative results for convex supremum-based surrogates.
\begin{restatable}{theorem}{LinearSupConvex}
\label{Thm:linear_sup_convex}
Let $\phi$ be convex and non-increasing margin-based loss, consider the surrogate loss defined by $\tilde{\phi}(f,\bx,y)=\sup_{\bx'\colon \|\bx-\bx'\|\leq \gamma}\phi(y f(\bx'))$. Then
\begin{enumerate}
      \item $\tilde{\phi}$ is not $\sH_{\mathrm{lin}}$-calibrated with respect to $\ell_{\gamma}$;      
      \item Given a non-decreasing and continuous function $g$ such that $g(-\gamma)+\C>0$ and  $g(\gamma) -\C<0$. Then $\tilde{\phi}$ is not $\sH_g$-calibrated with respect to $\ell_{\gamma}$; Specifically, if $G>\gamma$, $\tilde{\phi}$ is not $\sH_{\mathrm{relu}}$-calibrated with respect to $\ell_{\gamma}$.
\end{enumerate}
\end{restatable}
\begin{restatable}{theorem}{CalibrationSupConvex}
\label{Thm:calibration_sup_convex}
   Let $\sH$ be a hypothesis set containing $0$ that is regular for adversarial calibration. If a margin-based loss $\phi$ is convex and non-increasing, then the surrogate loss defined by $	\tilde{\phi}(f,\bx,y)=\sup_{\bx'\colon \|\bx-\bx'\|\leq \gamma}\phi(y f(\bx'))$ is not $\sH$-calibrated with respect to $\ell_{\gamma}$.
\end{restatable}
The proofs of Theorem~\ref{Thm:linear_sup_convex} and Theorem~\ref{Thm:calibration_sup_convex} are also included in Appendix~\ref{app:calibration_convex}. Since $\sH_{\mathrm{NN}}$ and $\sH_{\mathrm{all}}$ both contain $0$ and are regular for adversarial calibration, Theorem~\ref{Thm:calibration_sup_convex} leads to the following corollary.
\begin{corollary}
\label{Corollary:calibration_sup_convex}
Let $\phi$ be convex and non-increasing margin-based loss, consider the surrogate loss defined by $\tilde{\phi}(f,\bx,y)=\sup_{\bx'\colon \|\bx-\bx'\|\leq \gamma}\phi(y f(\bx'))$. Then
  \begin{enumerate}
      \item $\tilde{\phi}$ is not $\sH_{\mathrm{NN}}$-calibrated with respect to $\ell_{\gamma}$;  
      \item $\tilde{\phi}$ is not $\sH_{\mathrm{all}}$-calibrated with respect to $\ell_{\gamma}$.
  \end{enumerate}
\end{corollary}
Corollary~\ref{Corollary:calibration_margin_convex}, Theorem~\ref{Thm:linear_sup_convex}, Theorem~\ref{Thm:calibration_sup_convex} and Corollary~\ref{Corollary:calibration_sup_convex} above are stronger versions of the negative calibration results Theorem~10, Corollary~11, Theorem~12 and Corollary~13 in \citep{PNAMY21}, since the calibration considered in \citep{PNAMY21} is uniform calibration \citep[Definition~2.15]{steinwart2007compare}, which is stronger than non-uniform calibration \citep[Definition~2.7]{steinwart2007compare} considered in this work.

\subsection{Positive results}
\label{sec:calibration_positive}

In this section, we provide alternative surrogate losses that 
are $\sH$-calibrated with respect to $\ell_{\gamma}$. These 
results are similar but more general than their counterparts
in \citep{PNAMY21}, 

\subsubsection{Margin-based losses}
\label{sec:quasi-concave}

In light of the negative results of Section~\ref{sec:calibration_negative}, 
to find calibrated surrogate losses for adversarially robust classification, we need
to consider non-convex ones. One possible candidate is the family of
\emph{quasi-concave even} losses introduced by \citep[Definition 10]{pmlr-v125-bao20a}. Theorem~\ref{Thm:calibration_positive_linear} below is a correction to the main positive result, Theorem~11 in \citep{pmlr-v125-bao20a}, where we prove the theorem under the correct calibration definition.
\begin{restatable}{theorem}{CalibrationLinearPositive}
\label{Thm:calibration_positive_linear}
Let a margin-based loss $\phi$ be bounded, continuous, non-increasing, and quasi-concave even. Assume that $\phi(-t)>\phi(t)$ for any $\gamma< t \leq 1$. Then $\phi$ is $\sH_{\mathrm{lin}}$-calibrated with respect to $\ell_{\gamma}$ if and only if for any $\gamma< t \leq 1$,
  \begin{align}
  \label{eq:calibration_positive_linear}
  \phi(\gamma)+\phi(-\gamma) > \phi(t)+\phi(-t)\,.
 \end{align}
 \end{restatable}
 The proof of Theorem~\ref{Thm:calibration_positive_linear} is included in Appendix~\ref{app:sup_rhomargin_calibrate_general}, where we make use of Lemma~\ref{lemma:equivalent1_sup}, which is powerful since it applies to any symmetric hypothesis sets. Note Theorem~11 in \citep{pmlr-v125-bao20a} does not hold any more under the correct calibration Definition~\ref{def:H-calibration_true}, since their condition $\phi(\gamma)+\phi(-\gamma)>\phi(1)+\phi(-1)$ is much weaker than \eqref{eq:calibration_positive_linear}.

We next extend the above to show that under certain conditions, quasi-concave even surrogate losses are $\sH_g$-calibrated for the class of generalized linear models with respect to the adversarial $0/1$ loss. 
\begin{restatable}{theorem}{QuasiconcaveCalibrateGeneral}
\label{Thm:quasiconcave_calibrate_general}
  Let $g$ be a non-decreasing and continuous function such that $g(1+\gamma)< \C$ and  $g(-1-\gamma)> -\C$ for some $\C \geq 0$. Let a margin-based loss $\phi$ be bounded, continuous, non-increasing, and quasi-concave even. Assume that $\phi(g(-t)-\C)>\phi(\C-g(-t))$ and $g(-t)+g(t)\geq0$ for any $0\leq t \leq 1$. Then $\phi$ is $\sH_g$-calibrated with respect to $\ell_{\gamma}$ if and only if for any $0\leq t \leq 1$,
  \begin{align*}
  \phi(\C-g(-t))+\phi(g(-t)-\C) & = \phi(g(t)+\C)+\phi(-g(t)-\C)\\
  \text{and} \quad
  \min\curl*{\phi(\overline{A}(t))+\phi(-\overline{A}(t)),\phi(\underline{A}(t))+\phi(-\underline{A}(t)) } & > \phi(\C-g(-t))+\phi(g(-t)-\C),
 \end{align*}
  where $\overline{A}(t)=\max_{s\in[-t,t]}g(s)-g(s-\gamma)$ and $\underline{A}(t)=\min_{s\in[-t,t]}g(s)-g(s+\gamma)$.
\end{restatable}
The proof of Theorem~\ref{Thm:quasiconcave_calibrate_general} is included in Appendix~\ref{app:quasiconcave_calibrate_general}.
Specifically, when $g=()_{+}$, by Theorem~\ref{Thm:quasiconcave_calibrate_general}, we obtain the following corollary for $\sH_{\mathrm{relu}}$ by using the fact that $\phi(t)+\phi(-t)\geq \phi(\gamma)+\phi(-\gamma)$ when $0\leq t\leq\gamma$ by Part~\ref{part2_lemma:quasiconcave_even} of Lemma~\ref{lemma:quasiconcave_even}. Note when $g=()_{+}$,
\begin{align*}
&\overline{A}(t)=\max_{s\in[-t,t]}(s)_{+}-(s-\gamma)_{+}=
\begin{cases}
	t, 0\leq t< \gamma,\\
	\gamma, \gamma \leq t \leq 1.
\end{cases}\\
&\underline{A}(t)=\min_{s\in[-t,t]}(s)_{+}-g(s+\gamma)_{+}=-\gamma.
\end{align*}

\begin{corollary}
\label{Corollary:quasiconcave_calibrate_relu}
Assume that $\C>1+\gamma$. Let a margin-based loss $\phi$ be bounded, continuous, non-increasing, and quasi-concave even. Assume that $\phi(-\C)>\phi(\C)$. Then $\phi$ is $\sH_{\mathrm{relu}}$-calibrated with respect to $\ell_{\gamma}$ if and only if for any $0\leq t \leq 1$, 
 \begin{align*}
  \phi(\C)+\phi(-\C)=\phi(t+\C)+\phi(-t-\C)
  \quad \text{and} \quad
  \phi(\gamma)+\phi(-\gamma)>\phi(\C)+\phi(-\C).
  \end{align*}
\end{corollary}

In order to demonstrate the applicability of Theorem~\ref{Thm:calibration_positive_linear}, Theorem~\ref{Thm:quasiconcave_calibrate_general} and Corollary~\ref{Corollary:quasiconcave_calibrate_relu}, we consider a specific surrogate loss namely the \textit{$\rho$-margin loss} $\phi_{\rho}(t)=\min\curl*{1,\max\curl*{0,1-\frac{t}{\rho}}},~\rho>0$, 
which is a generalization of the ramp loss (see, for example, \citet{MohriRostamizadehTalwalkar2018}). Using Theorem~\ref{Thm:calibration_positive_linear}, Theorem~\ref{Thm:quasiconcave_calibrate_general} and Corollary~\ref{Corollary:quasiconcave_calibrate_relu}, we can conclude that the $\rho$-margin loss is calibrated under reasonable conditions for linear hypothesis sets and non-decreasing $g$-based hypothesis sets, since $\phi_{\rho}(t)$ is bounded, non-increasing and quasi-concave even. This is stated formally below.
\begin{theorem}
\label{theorem:rho_margin_calibtartion}
 Consider $\rho$-margin loss $\phi_{\rho}(t)=\min\curl*{1,\max\curl*{0,1-\frac{t}{\rho}}},~\rho>0$. Then,
  \begin{enumerate}
      \item $\phi_{\rho}$ is $\sH_{\mathrm{lin}}$-calibrated with respect to $\ell_{\gamma}$ if and only if $\rho>1$.
      \item Given a non-decreasing and continuous function $g$ such that $g(1+\gamma)< \C$ and  $g(-1-\gamma)> -\C$ for some $\C \geq 0$. Assume that $g(-t)+g(t)\geq0$ for any $0\leq t \leq 1$. Then $\phi_{\rho}$ is $\sH_g$-calibrated with respect to $\ell_{\gamma}$ if and only if for any $0\leq t \leq 1$,
  \begin{align*}
      \phi_{\rho}(\C-g(-t)) & = \phi_{\rho}(g(t)+\C)
  \quad \text{and} \quad
  \min\curl*{\phi_{\rho}(\overline{A}(t)),\phi_{\rho}(-\underline{A}(t)) }  > \phi_{\rho}(\C-g(-t)),
  \end{align*}
  where $\overline{A}(t)=\max_{s\in[-t,t]}g(s)-g(s-\gamma)$ and $\underline{A}(t)=\min_{s\in[-t,t]}g(s)-g(s+\gamma)$.
  
      \item Assume that $\C>1+\gamma$. Then $\phi_{\rho}$ is $\sH_{\mathrm{relu}}$-calibrated with respect to $\ell_{\gamma}$ if and only if $G\geq \rho >\gamma$.
  \end{enumerate}
\end{theorem}
Theorem~\ref{theorem:rho_margin_calibtartion} is a strict generalization of the positive calibration results in \citep{PNAMY21} for $\sH_g$ and $\sH_{\mathrm{relu}}$ where the authors require $G$ to be unbounded. By working with the weaker notion of non-uniform calibration, we avoid such a restriction on $G$.
% with the extra assumption $\C$ \citep[Corollary 19]{PNAMY21}, since our results hold without the unboundedness assumption for parameter $\C=\infty$ of the hypothesis sets $\sH_g$ and $\sH_{\mathrm{relu}}$.

\subsubsection{Supremum-based margin losses}
\label{sec:sup_margin_loss}

Recall that in Theorem~\ref{Thm:calibration_sup_convex}
we ruled out the possibility of finding $\sH$-calibrated supremum-based convex surrogate losses with respect to the adversarial $0/1$ loss. However, we show that the supremum-based $\rho$-margin loss is indeed $\sH$-calibrated. We state the calibration result below and present the proof in Appendix~\ref{app:sup_rhomargin_calibrate_general}.
\begin{restatable}{theorem}{SupRhomarginCalibrateGeneral}
\label{Thm:sup_rhomargin_calibrate_general}
  Consider $\rho$-margin loss $\phi_{\rho}(t)=\min\curl*{1,\max\curl*{0,1-\frac{t}{\rho}}},\rho>0$. Let $\sH$ be a symmetric hypothesis set, then the surrogate loss $\tilde{\phi}_{\rho}(f,\bx,y)=\sup_{\bx'\colon \|\bx-\bx'\|\leq \gamma}\phi_{\rho}(y f(\bx'))$ is $\sH$-calibrated with respect to $\ell_{\gamma}$.
  \end{restatable}
By Theorem~\ref{Thm:sup_rhomargin_calibrate_general}, we obtain the following corollary, since $\sH_{\mathrm{lin}}$, $\sH_{\mathrm{NN}}$ and $\sH_{\mathrm{all}}$ are all symmetric.
\begin{corollary}
\label{Corollary:sup_rhomargin_calibrate_general}
Consider $\rho$-margin loss $\phi_{\rho}(t)=\min\curl*{1,\max\curl*{0,1-\frac{t}{\rho}}},\rho>0$. Let $\tilde{\phi}_{\rho}(f,\bx,y)=\sup_{\bx'\colon \|\bx-\bx'\|\leq \gamma}\phi_{\rho}(y f(\bx'))$ be the surrogate loss. Then,
  \begin{enumerate}
      \item $\tilde{\phi}_{\rho}$ is $\sH_{\mathrm{lin}}$-calibrated with respect to $\ell_{\gamma}$;      
      \item $\tilde{\phi}_{\rho}$ is $\sH_{\mathrm{NN}}$-calibrated with respect to $\ell_{\gamma}$;  
      \item $\tilde{\phi}_{\rho}$ is $\sH_{\mathrm{all}}$-calibrated with respect to $\ell_{\gamma}$.
  \end{enumerate}
\end{corollary}
2.\ of Corollary~\ref{Corollary:sup_rhomargin_calibrate_general} is a strict generalization of the positive calibration  result in \citep{PNAMY21} for $\sH_{\mathrm{NN}}$ where the authors require $\Lambda$ to be unbounded. By working with the weaker notion of non-uniform calibration, we avoid such a restriction on $\Lambda$.
% with the extra assumption $\Lambda=\infty$ \citep[Corollary 21]{PNAMY21}, since our results hold without the unboundedness assumption for parameter $\Lambda$ of the hypothesis sets $\sH_{\mathrm{NN}}$.

\section{$\sH$-consistency}
\label{sec:consistency}

Next, we study the implications of our positive results for non-uniform calibration for establishing $\sH$-consistency. As discussed in Section~\ref{sec:introduction}, \citet{steinwart2007compare} showed that if $\ell_1$ is $\sH$-calibrated (it suffices to satisfy non-uniform calibration, that is condition~\eqref{eq:H-calibration_true}) with respect to $\ell_2$,  then $\sH$-consistency, that is condition~\eqref{eq:H-consistency}, holds for any probability distribution verifying the 
additional condition of
\emph{$\sP$-minimizability} \citep[Definition~2.4]{steinwart2007compare}. Although the $\sP$-minimizability condition is naturally satisfied and $\sH$-calibration often is a sufficient condition for $\sH$-consistency in the standard classification setting when considering the family of all measurable functions \citep[Theorem~3.2]{steinwart2007compare}, \citet{PNAMY21} point out that the adversarial loss presents new challenges when dealing with $\sP$-minimizability and requires carefully distinguishing among calibration and consistency to avoid drawing false conclusions.

Moreover, \citet{PNAMY21} show that the $\sH$-calibrated losses are $\sH$-consistent under certain conditions. %along with the realizability assumption. 
Analogously, in this section, we make use of \citep[Theorem~25, Theorem~27]{PNAMY21} to conclude that the $\sH$-calibrated losses studied in previous sections are $\sH$-consistent under the same conditions.

\begin{theorem}[Theorem~25 in \citep{PNAMY21}]
\label{Thm:calibrate_consistent_nonsup}
Let $\sP$ be a distribution over $\sX \times \sY$ and $\sH$ a hypothesis set for which $\cR^*_{\ell_{\gamma}, \sH}=0$. Let $\phi$ be a margin-based loss. If for $\eta \geq 0$, there exists $f^* \in \sH \subset \sH_{\mathrm{all}}$ such that $\cR_{\phi}(f^*)\leq \cR^*_{\phi, \sH_{\mathrm{all}}} + \eta< +\infty$ and $\phi$ is $\sH$-calibrated with respect to $\ell_{\gamma}$, then for all $\epsilon > 0$ there exists $\delta > 0$ such that for all $f\in\sH$ we have
    \[
        \cR_{\phi}(f)+\eta <\cR_{\phi,\sH}^*+\delta \implies \cR_{\ell_{\gamma}}(f)<\cR_{\ell_{\gamma},\sH}^*+\epsilon.
    \]
\end{theorem}

\begin{theorem}[Theorem~27 in \citep{PNAMY21}]
\label{Thm:calibrate_consistent_sup}
Given a distribution $\sP$ over $\sX\times\sY$ and a hypothesis set $\sH$ such that $\cR^*_{\ell_{\gamma},\sH}=0$. Let $\phi$ be a non-increasing margin-based loss. If there exists $f^*\in\sH\subset\sH_{\mathrm{all}}$ such that $\cR_{\phi}(f^*)=\cR^*_{\phi,\sH_{\mathrm{all}}}<\infty$ and $\tilde{\phi}(f,\bx,y)=\sup_{\bx'\colon \|\bx-\bx'\|\leq \gamma}\phi(y f(\bx'))$ is $\sH$-calibrated with respect to $\ell_{\gamma}$, then for all $\epsilon>0$ there exists $\delta>0$ such that for all $f\in\sH$ we have
    \[
        \cR_{\tilde{\phi}}(f) <\cR_{\tilde{\phi},\sH}^*+\delta \implies \cR_{\ell_{\gamma}}(f)<\cR_{\ell_{\gamma},\sH}^*+\epsilon.
    \]
\end{theorem}
Using Theorem~\ref{theorem:rho_margin_calibtartion} in Section~\ref{sec:quasi-concave} and Theorem~\ref{Thm:calibrate_consistent_nonsup} above, we conclude that the calibrated $\rho$-margin loss in Section~\ref{sec:quasi-concave} is consistent with respect to $\ell_{\gamma}$ for all distributions that satisfy the  realizability assumption, i.e., $\cR^*_{\ell_{\gamma}, \sH}=0$.
\begin{theorem}
\label{Thm:rho_margin_consistent}
  Consider the $\rho$-margin loss $\phi_{\rho}(t)=\min\curl*{1,\max\curl*{0,1-\frac{t}{\rho}}},~\rho>0$. Then,
  \begin{enumerate}
      \item If $\rho>1$, then $\phi_{\rho}$ is $\cH_{\mathrm{lin}}$-consistent wrt $\ell_{\gamma}$ for all distribution $P$ over $\cX\times\cY$ that satisfies $\cR^*_{\ell_{\gamma},\cH_{\mathrm{lin}}}=0$ and there exists $f^*\in\cH_{\mathrm{lin}}$ such that $\cR_{\phi_{\rho}}(f^*)= \cR^*_{\phi_{\rho},\cH_{\mathrm{all}}}<\infty$.
      \item 
      Given a non-decreasing and continuous function $g$ such that $g(1+\gamma)< \C$ and  $g(-1-\gamma)> -\C$ for some $\C \geq 0$. Assume that $g(-t)+g(t)\geq0$ for any $0\leq t \leq 1$. Let $\overline{A}(t)=\max_{s\in[-t,t]}g(s)-g(s-\gamma)$ and $\underline{A}(t)=\min_{s\in[-t,t]}g(s)-g(s+\gamma)$ for any $0\leq t\leq 1$. If for any $0\leq t \leq 1$,
  $ \phi_{\rho}(\C-g(-t)) = \phi_{\rho}(g(t)+\C)
  \text{ and }
  \min\curl*{\phi_{\rho}(\overline{A}(t)),\phi_{\rho}(-\underline{A}(t)) }  > \phi_{\rho}(\C-g(-t)),$
      then $\phi_{\rho}$ is $\cH_g$-consistent wrt $\ell_{\gamma}$ for all distribution $P$ over $\cX\times\cY$ that satisfies $\cR^*_{\ell_{\gamma},\cH_{g}}=0$ and there exists $f^*\in\cH_{g}$ such that $\cR_{\phi_{\rho}}(f^*)= \cR^*_{\phi_{\rho},\cH_{\mathrm{all}}}<\infty$.
      
      \item If $\C>1+\gamma$ and $\C\geq\rho>\gamma$, then $\phi_{\rho}$ is $\cH_{\mathrm{relu}}$-consistent wrt $\ell_{\gamma}$ for all distribution $P$ over $\cX\times\cY$ that satisfies $\cR^*_{\ell_{\gamma},\cH_{\mathrm{relu}}}=0$ and there exists $f^*\in\cH_{\mathrm{relu}}$ such that $\cR_{\phi_{\rho}}(f^*)= \cR^*_{\phi_{\rho},\cH_{\mathrm{all}}}<\infty$.
  \end{enumerate}
\end{theorem}
Using Theorem~\ref{Thm:sup_rhomargin_calibrate_general} in Section~\ref{sec:sup_margin_loss} and Theorem~\ref{Thm:calibrate_consistent_sup}, we conclude that the calibrated supremum-based $\rho$-margin loss in Section~\ref{sec:sup_margin_loss} is also consistent wrt $\ell_{\gamma}$ for all distributions that satisfy realizability assumptions.
\begin{theorem}
\label{Thm:sup_rho_margin_consistent}
Consider $\rho$-margin loss $\phi_{\rho}(t)=\min\curl*{1,\max\curl*{0,1-\frac{t}{\rho}}},\rho>0$. Let $\sH$ be a symmetric hypothesis set, then the surrogate loss $\tilde{\phi}_{\rho}(f,\bx,y)=\sup_{\bx'\colon \|\bx-\bx'\|\leq \gamma}\phi_{\rho}(y f(\bx'))$ is $\sH$-consistent with respect to $\ell_{\gamma}$ for all distributions $\sP$ over $\cX\times\cY$ that satisfy: $\cR^*_{\ell_{\gamma},\cH}=0$ and there exists $f^*\in\cH$ such that $\cR_{\phi_{\rho}}(f^*)= \cR^*_{\phi_{\rho},\cH_{\mathrm{all}}}<\infty$.

\end{theorem}

\section{Conclusion}
\label{sec:conclusion}

We presented a careful analysis of the $\sH$-calibration of surrogate losses, including a series of 
negative results for surrogate losses commonly used in practice, as well as a number of positive
results for surrogate losses that we prove additionally to be $\sH$-consistent, 
provided that some other natural conditions hold. 
Our results significantly extend previously known results and provide a solid guidance for the design
of algorithms for adversarial robustness with theoretical guarantees. Moreover, several of our proof techniques
for calibration and consistency can further be relevant to the analysis of other loss functions.

%Acknowledgments---Will not appear in anonymized version
\acks{We warmly thank our colleague Natalie Frank for discussions and our previous joint work on this topic.}

\bibliography{arxiv2}

\newpage
\appendix

\renewcommand{\contentsname}{Contents of Appendix}
\tableofcontents
\addtocontents{toc}{\protect\setcounter{tocdepth}{3}} 
\clearpage

\section{Deferred Proofs}
For convenience, let $\Delta\cC_{\ell,\sH}(f,\bx,\eta)\colon=\cC_{\ell}(f,\bx,\eta)-\cC_{\ell,\sH}^*(\bx,\eta)$, $\underline{M}(f,\bx,\gamma)\colon=\inf_{\bx'\colon \|\bx - \bx'\|\leq\gamma} f(\bx')$
and $
\overline{M}(f,\bx,\gamma)\colon= -\inf_{\bx'\colon \|\bx - \bx'\|\leq\gamma} -f(\bx')=\sup_{\bx'\colon \|\bx - \bx'\|\leq\gamma} f(\bx').$

\subsection{Proof of 
Theorem~\ref{Thm:calibration_margin_convex}, Theorem~\ref{Thm:linear_sup_convex} and Theorem~\ref{Thm:calibration_sup_convex}}
\label{app:calibration_convex}
We first characterize the calibration function $\delta_{\max}(\epsilon,\bx,\eta)$ of losses $(\ell, \ell_{\gamma})$ at $\eta=\frac12$, $\epsilon=\frac12$ and distinguishing $\bx_0\in \sX$ given a hypothesis set $\sH$ which is regular for adversarial calibration.
\begin{lemma}
\label{lemma:distinguishing}
Let $\sH$ be a hypothesis set that is regular for adversarial calibration. For distinguishing $\bx_0\in \sX$, the calibration function $\delta_{\max}(\epsilon,\bx,\eta)$ of losses $(\ell, \ell_{\gamma})$ satisfies
\begin{align*}
\delta_{\max}\left(\frac12,\bx_0,\frac12\right)=\inf_{f\in\sH\colon ~\underline{M}(f,\bx_0,\gamma)\leq 0 \leq \overline{M}(f,\bx_0,\gamma)}\Delta\cC_{\ell,\sH}(f,\bx_0,\frac12).
\end{align*}
\end{lemma}
\begin{proof}
By the definition of inner risk \eqref{eq:conditional-risk} and adversarial 0-1 loss $\ell_{\gamma}$ \eqref{eq:supinf01}, the inner $\ell_{\gamma}$-risk is
\begin{align*}
  \cC_{\ell_{\gamma}}(f,\bx,\eta)
  &=\eta \mathds{1}_{\left\{\underline{M}(f,\bx,\gamma)\leq 0\right\}}+(1-\eta) \mathds{1}_{\left\{\overline{M}(f,\bx,\gamma)\geq 0\right\}}\\
  &=\begin{cases}
   1 & \text{if} ~ \underline{M}(f,\bx,\gamma)\leq 0 \leq \overline{M}(f,\bx,\gamma),\\
   \eta & \text{if} ~ \overline{M}(f,\bx,\gamma)<0,\\
   1-\eta & \text{if} ~ \underline{M}(f,\bx,\gamma)> 0.\\
  \end{cases}
\end{align*}
For distinguishing $\bx_0$ and $\eta\in [0,1]$, $\{f\in \sH: \overline{M}(f,\bx_0,\gamma)\}<0\}$ and $\{f\in \sH: \underline{M}(f,\bx_0,\gamma)>0\}$ are not empty sets. Thus
\begin{align*}
     \cC^*_{\ell_{\gamma},\sH}(\bx_0,\eta)=\inf_{f\in\sH}\cC_{\ell_{\gamma}}(f,\bx_0,\eta)=\min\curl*{\eta,1-\eta}\,.
\end{align*}
% Since $\sH$ contains $0$, $\{f\in \sH: \underline{M}(f,\bx_0,\gamma)\leq 0 \leq \overline{M}(f,\bx_0,\gamma)\}$ is also not empty set. 
Note for $f\in\{f\in \sH: \underline{M}(f,\bx_0,\gamma)\leq 0 \leq \overline{M}(f,\bx_0,\gamma)\}$, $\Delta\cC_{\ell_{\gamma},\sH}(f,\bx_0,\eta)=\max\curl*{\eta,1-\eta}$; for $f\in\{f\in \sH: \overline{M}(f,\bx_0,\gamma)\}<0\}$, $\Delta\cC_{\ell_{\gamma},\sH}(f,\bx_0,\eta)=\eta-\min\curl*{\eta,1-\eta}=\max\curl*{0,2\eta-1}=|2\eta-1|\mathds{1}_{(2\eta-1)(\underline{M}(f,\bx_0,\gamma))\leq 0}$ since $\underline{M}(f,\bx_0,\gamma)\leq \overline{M}(f,\bx_0,\gamma)<0$; for $f\in\{f\in \sH: \underline{M}(f,\bx_0,\gamma)>0\}$, $\Delta\cC_{\ell_{\gamma},\sH}(f,\bx_0,\eta)=(1-\eta)-\min\curl*{\eta,1-\eta}=\max\curl*{0,1-2\eta}=|2\eta-1|\mathds{1}_{(2\eta-1)(\underline{M}(f,\bx_0,\gamma))\leq 0}$. Therefore,
\begin{align*}
  \Delta\cC_{\ell_{\gamma},\sH}(f,\bx_0,\eta)=
    \begin{cases}
    \max\curl*{\eta,1-\eta} & \text{if} ~ \underline{M}(f,\bx_0,\gamma)\leq 0 \leq \overline{M}(f,\bx_0,\gamma),\\
   |2\eta-1|\mathds{1}_{(2\eta-1)(\underline{M}(f,\bx_0,\gamma))\leq 0} & \text{if} ~ \underline{M}(f,\bx_0,\gamma)>0 \text{~or~} \overline{M}(f,\bx_0,\gamma)<0.\\
    \end{cases}
\end{align*}
By \eqref{eq:def-calibration-function-true}, for a fixed $\eta\in[0,1]$ and $\bx \in \sX$, the calibration function of losses $(\ell, \ell_{\gamma})$ is 
\begin{align*}
    \delta_{\max}(\epsilon,\bx,\eta)=\inf_{f\in\sH}\curl*{\Delta\cC_{\ell,\sH}(f,\bx,\eta) \mid \Delta\cC_{\ell_{\gamma},
    \sH}(f,\bx,\eta)\geq\epsilon }\,.
\end{align*}
Observe that for all $\eta\in [0,1]$,
\begin{align}
\label{eq:observation}
\max\curl*{\eta,1-\eta}= \frac{1}{2}[(1-\eta ) + \eta + |(1-\eta) - \eta|]= \frac{1}{2}[1 + |2 \eta - 1|]\geq |2 \eta - 1|.   
\end{align}
For distinguishing $\bx_0$, $\eta=\frac12$ and $\epsilon=\frac12$, $\Delta\cC_{\ell_{\gamma},\sH}(f,\bx_0,\frac12)\geq\frac12$ if and only if $\underline{M}(f,\bx_0,\gamma)\leq 0 \leq \overline{M}(f,\bx_0,\gamma)$ since $|2\eta-1|<\epsilon\leq\max\curl*{\eta,1-\eta}$. Therefore, 
\begin{align*}
\delta_{\max}\left(\frac12,\bx_0,\frac12\right)=\inf_{f\in\sH\colon ~\underline{M}(f,\bx_0,\gamma)\leq 0 \leq \overline{M}(f,\bx_0,\gamma)}\Delta\cC_{\ell,\sH}(f,\bx_0,\frac12).   
\end{align*}
\end{proof}
\CalibrationMarginConvex*

\begin{proof}
By Lemma~\ref{lemma:distinguishing}, for distinguishing $\bx_0\in \sX$, the calibration function $\delta_{\max}(\epsilon,\bx,\eta)$ of losses $(\phi, \ell_{\gamma})$ satisfies
\begin{align*}
\delta_{\max}\left(\frac12,\bx_0,\frac12\right)=\inf_{f\in\sH\colon ~\underline{M}(f,\bx_0,\gamma)\leq 0 \leq \overline{M}(f,\bx_0,\gamma)}\Delta\cC_{\phi,\sH}(f,\bx_0,\frac12).
\end{align*}
Suppose that $\phi$ is $\sH$-calibrated with respect to $\ell_{\gamma}$. By Proposition~\ref{prop:calibration_function_positive}, $\phi$ is $\sH$-calibrated with respect to $\ell_{\gamma}$ if and only if its calibration function $\delta_{\max}$ satisfies $\delta_{\max}(\epsilon,\bx,\eta)>0$ for all $\bx\in \sX$, $\eta\in
     [0,1]$ and $\epsilon>0$. In particular, the condition requires $\delta_{\max}\left(\frac12,\bx_0,\frac12\right)>0$, that is,
\begin{align*}
    	\inf_{f\in\sH\colon~\underline{M}(f,\bx_0,\gamma)\leq 0 \leq \overline{M}(f,\bx_0,\gamma)}\Delta\cC_{\phi,\sH}(f,\bx_0,\frac12)>0,
\end{align*}
which is equivalent to
\begin{align}
\inf_{f\in\sH\colon~\underline{M}(f,\bx_0,\gamma)\leq 0 \leq \overline{M}(f,\bx_0,\gamma)}\cC_{\phi}(f,\bx_0,\frac12)>
     \inf_{f\in\sH}\cC_{\phi}(f,\bx_0,\frac12)\,,
     \label{eq:larger_adv_margin}
\end{align}
By the definition of inner risk \eqref{eq:conditional-risk},
\begin{align}
   \cC_{\phi}(f,\bx_0,\frac12)=\frac12 (\phi(f(\bx_0))+\phi(-f(\bx_0)))\,.
\label{eq:phi_CCR_adv_margin} 
\end{align}
Since $\phi$ is convex, by Jensen's inequality, for any $f\in \sH$, the
following holds:
\begin{align*}
   \cC_{\phi}(f,\bx_0,\frac12)\geq \phi\left(\frac12 f(\bx_0) -\frac12 f(\bx_0) \right)= \phi(0).
\end{align*}
For $f=f_0$, we have $f_0(\bx_0)=0$ and by \eqref{eq:phi_CCR_adv_margin},
   \begin{align*}
\cC_{\phi}(f_0,\bx_0,\frac12)=\frac12(\phi(0)+\phi(0))=\phi(0)\,.
\end{align*}
Moreover, when $f=f_0$, $\underline{M}(f_0,\bx_0,\gamma)\leq f_0(\bx_0)=0\leq\overline{M}(f_0,\bx_0,\gamma)$.
Thus
\begin{align*}
    \inf_{f\in\sH\colon ~\underline{M}(f,\bx_0,\gamma)\leq 0 \leq \overline{M}(f,\bx_0,\gamma)}\cC_{\phi}(f,\bx_0,\frac12)=
     \inf_{f\in\sH}\cC_{\phi}(f,\bx_0,\frac12)=\phi(0)\,,
\end{align*}
where the minimum can be achieved by $f=f_0$, contradicting \eqref{eq:larger_adv_margin}. Therefore, $\phi$ is not $\sH$-calibrated with respect to $\ell_{\gamma}$.
\end{proof}

\LinearSupConvex*
\begin{proof}
By Lemma~\ref{lemma:distinguishing}, for distinguishing $\bx_0\in \sX$, the calibration function $\delta_{\max}(\epsilon,\bx,\eta)$ of losses $(\tilde{\phi}, \ell_{\gamma})$ satisfies
\begin{align*}
\delta_{\max}\left(\frac12,\bx_0,\frac12\right)=\inf_{f\in\sH\colon ~\underline{M}(f,\bx_0,\gamma)\leq 0 \leq \overline{M}(f,\bx_0,\gamma)}\Delta\cC_{\tilde{\phi},\sH}(f,\bx_0,\frac12).
\end{align*}
Next we first consider the case where $\sH=\sH_{\mathrm{lin}}$. Take distinguishing $\bx_0\in \sX$ and $f_0\in \sH_{\mathrm{lin}}$ such that $f_0(\bx_0)=0$. 
As shown by \citet{awasthi2020adversarial}, for $f\in\sH_{\mathrm{lin}}=\curl*{\bx\rightarrow \bw \cdot \bx \mid \|\bw\|=1}$,
\begin{align*}
&\underline{M}(f,\bx,\gamma)=\inf_{\bx'\colon \|\bx - \bx'\|\leq\gamma} f(\bx')=\inf_{\bx'\colon \|\bx-\bx'\|\leq \gamma}(\bw \cdot \bx')=\bw \cdot \bx-\gamma \|\bw\|=f(\bx)-\gamma, \\   
&\overline{M}(f,\bx,\gamma)=-\inf_{\bx'\colon \|\bx - \bx'\|\leq\gamma} -f(\bx')=-\inf_{\bx'\colon \|\bx-\bx'\|\leq \gamma}(-\bw \cdot \bx')=\bw \cdot \bx+\gamma \|\bw\|=f(\bx)+\gamma.
\end{align*}
Suppose that $\tilde{\phi}$ is $\sH_{\mathrm{lin}}$-calibrated with respect to $\ell_{\gamma}$. By Proposition~\ref{prop:calibration_function_positive}, $\tilde{\phi}$ is $\sH_{\mathrm{lin}}$-calibrated with respect to $\ell_{\gamma}$ if and only if its calibration function $\delta_{\max}$ satisfies $\delta_{\max}(\epsilon,\bx,\eta)>0$ for all $\bx\in \sX$, $\eta\in[0,1]$ and $\epsilon>0$. In particular, the condition requires $\delta_{\max}\left(\frac12,\bx_0,\frac12\right)>0$, that is,
\begin{align*}
    	\inf_{f\in\sH_{\mathrm{lin}}\colon~-\gamma\leq f(\bx_0)\leq\gamma}\Delta\cC_{\tilde\phi,\sH_{\mathrm{lin}}}(f,\bx_0,\frac12)>0,
\end{align*}
which is equivalent to
\begin{align}
\inf_{f\in\sH_{\mathrm{lin}}\colon~-\gamma\leq f(\bx_0)\leq\gamma}\cC_{\tilde{\phi}}(f,\bx_0,\frac12)>\inf_{f\in\sH_{\mathrm{lin}}}\cC_{\tilde{\phi}}(f,\bx_0,\frac12)\,,
     \label{eq:larger_adv_linear}
\end{align}
By \eqref{eq:phi_CCR_adv_GN}, for $f\in\sH_{\mathrm{lin}}$,
\begin{align}
   \cC_{\tilde{\phi}}(f,\bx_0,\frac12)=\frac12 \phi(f(\bx_0)-\gamma)+\frac12\phi(-f(\bx_0)-\gamma)\,.
\label{eq:phi_CCR_adv_linear} 
\end{align}
Since $\phi$ is convex, by Jensen's inequality, for any $f\in \sH_{\mathrm{lin}}$, the
following holds:
\begin{align*}
   \cC_{\tilde{\phi}}(f,\bx_0,\frac12)\geq \phi\left(\frac12 (f(\bx_0)-\gamma) -\frac12 (f(\bx_0)+\gamma) \right)=\phi(-\gamma)\,.
\end{align*}
For $f=f_0$, we have $f_0(\bx_0)=0$ and by \eqref{eq:phi_CCR_adv_linear},
\begin{align*}
\cC_{\tilde{\phi}}(f_0,\bx_0,\frac12)=\frac12(\phi(-\gamma)+\phi(-\gamma))=\phi(-\gamma)\,.
\end{align*}
Moreover, when $f=f_0$, $-\gamma\leq f_0(\bx_0)=0\leq\gamma$.
Thus
\begin{align*}
    \inf_{f\in\sH\colon ~-\gamma\leq f(\bx_0)\leq\gamma}\cC_{\tilde{\phi}}(f,\bx_0,\frac12)=
     \inf_{f\in\sH}\cC_{\tilde{\phi}}(f,\bx_0,\frac12)=\phi(-\gamma)\,,
\end{align*}
where the minimum can be achieved by $f=f_0$, contradicting \eqref{eq:larger_adv_linear}. Therefore, $\tilde{\phi}$ is not $\sH_{\mathrm{lin}}$-calibrated with respect to $\ell_{\gamma}$.

Then we consider the case where $\sH=\sH_g$. By the assumption on $g$, $0\in \sX$ is distinguishing. As shown by \citet{awasthi2020adversarial}, for $f\in\sH_{g}$,
\begin{align*}
\underline{M}(f,\bx,\gamma)=g(\bw \cdot \bx -\gamma)+b, 
\quad\overline{M}(f,\bx,\gamma)=g(\bw \cdot \bx +\gamma)+b.
\end{align*}
Suppose that $\tilde{\phi}$ is $\sH_{g}$-calibrated with respect to $\ell_{\gamma}$.By Proposition~\ref{prop:calibration_function_positive}, $\tilde{\phi}$ is $\sH_{g}$-calibrated with respect to $\ell_{\gamma}$ if and only if its calibration function $\delta_{\max}$ satisfies $\delta_{\max}(\epsilon,\bx,\eta)>0$ for all $\bx\in \sX$, $\eta\in[0,1]$ and $\epsilon>0$. In particular, the condition requires $\delta_{\max}\left(\frac12,0,\frac12\right)>0$, that is,
\begin{align*}
    	\inf_{f\in\sH_{g}\colon~g(-\gamma)+b\leq 0\leq g(\gamma)+b}\Delta\cC_{\tilde\phi,\sH_g}(f,0,\frac12)>0,
\end{align*}
which is equivalent to
\begin{align}
\inf_{f\in\sH_g\colon~g(-\gamma)+b\leq 0\leq g(\gamma)+b}\cC_{\tilde{\phi}}(f,0,\frac12)>\inf_{f\in\sH_g}\cC_{\tilde{\phi}}(f,0,\frac12)\,,
     \label{eq:larger_adv_hg}
\end{align}
By \eqref{eq:phi_CCR_adv_GN}, for $f\in\sH_g$,
\begin{align}
   \cC_{\tilde{\phi}}(f,0,\frac12)=\frac12 \phi(g(-\gamma)+b)+\frac12\phi(-g(\gamma)-b)\,.
\label{eq:phi_CCR_adv_hg} 
\end{align}
Since $\phi$ is convex, by Jensen's inequality, for any $f\in \sH_g$, the
following holds:
\begin{align*}
   \cC_{\tilde{\phi}}(f,0,\frac12)\geq \phi\left(\frac12 (g(-\gamma)+b) +\frac12 (-g(\gamma)-b) \right)=\phi\left(\frac{g(-\gamma)-g(\gamma)}{2}\right)\,.
\end{align*}
Take $f_0\in \sH_g$ with $b_0=\frac{-g(\gamma)-g(-\gamma)}{2}$, we have $g(-\gamma)+b_0=-g(\gamma)-b_0=\frac{g(-\gamma)-g(\gamma)}{2}$ and by \eqref{eq:phi_CCR_adv_hg},
\begin{align*}
\cC_{\tilde{\phi}}(f_0,0,\frac12)=\frac12 \phi(g(-\gamma)+b_0)+\frac12\phi(-g(\gamma)-b_0)=\phi\left(\frac{g(-\gamma)-g(\gamma)}{2}\right)\,.
\end{align*}
Moreover, when $f=f_0$, $g(-\gamma)+b_0\leq0\leq g(\gamma)+b_0$.
Thus
\begin{align*}
    \inf_{f\in\sH_g\colon ~g(-\gamma)+b\leq 0\leq g(\gamma)+b}\cC_{\tilde{\phi}}(f,0,\frac12)=
     \inf_{f\in\sH_g}\cC_{\tilde{\phi}}(f,0,\frac12)=\phi\left(\frac{g(-\gamma)-g(\gamma)}{2}\right)\,,
\end{align*}
where the minimum can be achieved by $f=f_0$, contradicting \eqref{eq:larger_adv_hg}. Therefore, $\tilde{\phi}$ is not $\sH_g$-calibrated with respect to $\ell_{\gamma}$.
\end{proof}

\CalibrationSupConvex*
\begin{proof}
By Lemma~\ref{lemma:distinguishing}, for distinguishing $\bx_0\in \sX$, the calibration function $\delta_{\max}(\epsilon,\bx,\eta)$ of losses $(\tilde{\phi}, \ell_{\gamma})$ satisfies
\begin{align*}
\delta_{\max}\left(\frac12,\bx_0,\frac12\right)=\inf_{f\in\sH\colon ~\underline{M}(f,\bx_0,\gamma)\leq 0 \leq \overline{M}(f,\bx_0,\gamma)}\Delta\cC_{\tilde{\phi},\sH}(f,\bx_0,\frac12).
\end{align*}
Suppose that $\tilde{\phi}$ is $\sH$-calibrated with respect to $\ell_{\gamma}$. By Proposition~\ref{prop:calibration_function_positive}, $\tilde{\phi}$ is $\sH$-calibrated with respect to $\ell_{\gamma}$ if and only if its calibration function $\delta_{\max}$ satisfies $\delta_{\max}(\epsilon,\bx,\eta)>0$ for all $\bx\in \sX$, $\eta\in
     [0,1]$ and $\epsilon>0$. In particular, the condition requires $\delta_{\max}\left(\frac12,\bx_0,\frac12\right)>0$, that is,
\begin{align*}
    	\inf_{f\in\sH\colon~\underline{M}(f,\bx_0,\gamma)\leq 0 \leq \overline{M}(f,\bx_0,\gamma)}\Delta\cC_{\tilde\phi,\sH}(f,\bx_0,\frac12)>0,
\end{align*}
which is equivalent to
\begin{align}
\inf_{f\in\sH\colon~\underline{M}(f,\bx_0,\gamma)\leq 0 \leq \overline{M}(f,\bx_0,\gamma)}\cC_{\tilde{\phi}}(f,\bx_0,\frac12)>
     \inf_{f\in\sH}\cC_{\tilde{\phi}}(f,\bx_0,\frac12)\,,
     \label{eq:larger_adv_GN}
\end{align}
As shown by \citet{awasthi2020adversarial}, $\tilde{\phi}$ has the equivalent form
\begin{align*}
	\tilde{\phi}(f,\bx,y)=\phi\left(\inf\limits_{\| \bx' - \bx \| \leq \gamma }\left(yf(\bx')\right)\right)\,.
\end{align*}
By the definition of inner risk \eqref{eq:conditional-risk},
\begin{align}
   \cC_{\tilde{\phi}}(f,\bx_0,\frac12)=\frac12 (\phi(\underline{M}(f,\bx_0,\gamma))+\phi(-\overline{M}(f,\bx_0,\gamma)))\,.
\label{eq:phi_CCR_adv_GN} 
\end{align}
Since $\phi$ is convex, by Jensen's inequality, for any $f\in \sH$, the
following holds:
\begin{align*}
   \cC_{\tilde{\phi}}(f,\bx_0,\frac12)\geq \phi\left(\frac12 \underline{M}(f,\bx_0,\gamma) -\frac12 \overline{M}(f,\bx_0,\gamma) \right)=\phi\left(\frac12 (\underline{M}(f,\bx_0,\gamma)-\overline{M}(f,\bx_0,\gamma))\right)\geq \phi(0),
\end{align*}
where the last inequality used the fact that
\begin{align*}
    \frac12 (\underline{M}(f,\bx_0,\gamma)-\overline{M}(f,\bx_0,\gamma))\leq 0
\end{align*}
and $\phi$ is non-increasing.
For $f=0$, we have $\underline{M}(f,\bx_0,\gamma)=\overline{M}(f,\bx_0,\gamma)=0$ and by \eqref{eq:phi_CCR_adv_GN},
\begin{align*}
\cC_{\tilde{\phi}}(f,\bx_0,\frac12)=\frac12(\phi(0)+\phi(0))=\phi(0)\,.
\end{align*}
Moreover, when $\underline{M}(f,\bx_0,\gamma)=\overline{M}(f,\bx_0,\gamma)=0$, $\underline{M}(f,\bx_0,\gamma)\leq0\leq\overline{M}(f,\bx_0,\gamma)$ is satisfied.
Thus
\begin{align*}
    \inf_{f\in\sH\colon ~\underline{M}(f,\bx_0,\gamma)\leq 0 \leq \overline{M}(f,\bx_0,\gamma)}\cC_{\tilde{\phi}}(f,\bx_0,\frac12)=
     \inf_{f\in\sH}\cC_{\tilde{\phi}}(f,\bx_0,\frac12)=\phi(0)\,,
\end{align*}
where the minimum can be achieved by $f=0$, contradicting \eqref{eq:larger_adv_GN}. Therefore, $\tilde{\phi}$ is not $\sH$-calibrated with respect to $\ell_{\gamma}$.
\end{proof}

\subsection{Property of \texorpdfstring{$\bar{\cC}_{\phi}(t,\eta)$}{BarC}}
For a margin-based loss $\phi$, denote $\bar{\cC}_{\phi}(t,\eta)\colon=\eta\phi(t)+(1-\eta)\phi(-t)$ for any $\eta\in [0,1]$ and $t\in \mathbb{R}$. In this section, we characterize the property of $\bar{\cC}_{\phi}(t,\eta)$ when $\phi$ is bounded, continuous, non-increasing and quasi-concave even, which would be useful in the proof of Theorem~\ref{Thm:calibration_positive_linear} and Theorem~\ref{Thm:quasiconcave_calibrate_general}. Without loss of generality, assume that $g$ is continuous, non-decreasing and satisfies $g(-1-\gamma)+\C>0$, $g(1+\gamma)-\C<0$.
\begin{lemma}
\label{lemma:quasiconcave_even}
Let $\phi$ be a margin-based loss. If $\phi$ is bounded, continuous, non-increasing, quasi-concave even, then
\begin{enumerate}
    \item $\bar{\cC}_{\phi}(t,\eta)$ is quasi-concave in $t\in\mathbb{R}$ for all $\eta\in[0,1]$.
    \label{part1_lemma:quasiconcave_even}
    \item $\bar{\cC}_{\phi}(t,\frac12)$ is even and non-increasing in $t$ when $t\geq 0$.
    \label{part2_lemma:quasiconcave_even}
    \item For $l,u\in \mathbb{R}(l\leq u),~\inf_{t\in [l,u]}\bar{\cC}_{\phi}(t,\eta)=\min\curl*{\bar{\cC}_{\phi}(l,\eta),\bar{\cC}_{\phi}(u,\eta)}$ for all $\eta\in[0,1]$.
    \label{part3_lemma:quasiconcave_even}
    \item For all $\eta\in (\frac12,1]$, $\bar{\cC}_{\phi}(t,\eta)$ is non-increasing in $t$ when $t\geq 0$.
    \label{part4_lemma:quasiconcave_even}
    \item For all $\eta\in [0,\frac12)$, $\bar{\cC}_{\phi}(t,\eta)$ is non-decreasing in $t$ when $t\leq 0$.
    \label{part5_lemma:quasiconcave_even}
    \item If $\phi(-t)>\phi(t)$ for any $\gamma< t \leq 1$, then, for all $\eta\in (\frac12,1]$ and any $\gamma< t \leq 1$, $\bar{\cC}_{\phi}(-t,\eta)>\bar{\cC}_{\phi}(t,\eta)$.
    \label{part6_lemma:quasiconcave_even}
    \item If $\phi(-t)>\phi(t)$ for any $\gamma< t \leq 1$, then, for all $\eta\in [0,\frac12)$ and any $\gamma< t \leq 1$, $\bar{\cC}_{\phi}(-t,\eta)<\bar{\cC}_{\phi}(t,\eta)$.
    \label{part7_lemma:quasiconcave_even}
    \item If $\phi(g(-t)-\C)>\phi(\C-g(-t))$, $g(-t)+g(t)\geq0$ for any $0\leq t \leq 1$, then, for all $\eta\in (\frac12,1]$ and any $0\leq t \leq 1$, $\bar{\cC}_{\phi}(g(-t)-\C,\eta)>\bar{\cC}_{\phi}(g(t)+\C,\eta)$.
    \label{part8_lemma:quasiconcave_even}
    \item If $\phi(g(-t)-\C)>\phi(\C-g(-t))$, $g(-t)+g(t)\geq0$ for any $0\leq t \leq 1$, then, for any $0\leq t \leq 1$, $\bar{\cC}_{\phi}(g(-t)-\C,\eta)<\bar{\cC}_{\phi}(g(t)+\C,\eta) \text{ for all } \eta\in [0,\frac12) \text{~if and only if~} \phi(\C-g(-t))+\phi(g(-t)-\C)=\phi(g(t)+\C)+\phi(-g(t)-\C)$.
    \label{part9_lemma:quasiconcave_even}
 
\end{enumerate}
\end{lemma}
\begin{proof}
	Part~\ref{part1_lemma:quasiconcave_even},\ref{part2_lemma:quasiconcave_even},\ref{part4_lemma:quasiconcave_even} of Lemma~\ref{lemma:quasiconcave_even} are stated in \citep[Lemma~13]{pmlr-v125-bao20a}.
 Part~\ref{part3_lemma:quasiconcave_even} is a corollary of Part~\ref{part1_lemma:quasiconcave_even} by the characterization of continuous and quasi-convex functions in \citep{boyd2004convex}.

Consider Part~\ref{part5_lemma:quasiconcave_even}. For $\eta\in [0,\frac12)$, and $t_1,t_2\leq0$. Suppose that $t_1<t_2$, then 
\begin{align*}
    &\phi(t_1)-\phi(-t_1)-\phi(t_2)+\phi(-t_2)\\
    \geq& \phi(t_2)-\phi(-t_2)-\phi(t_2)+\phi(-t_2)\\
    =&0
\end{align*}
since $\phi$ is non-increasing.
By Part~\ref{part2_lemma:quasiconcave_even} of Lemma~\ref{lemma:quasiconcave_even}, $\phi(t)+\phi(-t)$ is non-decreasing in $t$ when $t\leq 0$. 
Therefore, for $\eta\in [0,\frac12)$,
\begin{align*}
    &\bar{\cC}_{\phi}(t_1,\eta) -  \bar{\cC}_{\phi}(t_2,\eta) \\
    =& (\phi(t_1)-\phi(-t_1)-\phi(t_2)+\phi(-t_2))\eta+\phi(-t_1)-\phi(-t_2)\\
    \leq& (\phi(t_1)-\phi(-t_1)-\phi(t_2)+\phi(-t_2))\frac12+\phi(-t_1)-\phi(-t_2)\\
    =& \frac12 (\phi(t_1)+\phi(-t_1)-\phi(t_2)-\phi(-t_2))\\
    \leq& 0.
\end{align*}
Consider Part~\ref{part6_lemma:quasiconcave_even}, For $\eta\in (\frac12,1]$ and any $\gamma< t \leq 1$,
\begin{align*}
	\bar{\cC}_{\phi}(-t,\eta)-\bar{\cC}_{\phi}(t,\eta)&=\eta\phi(-t)+(1-\eta)\phi(t)-\eta\phi(t)-(1-\eta)\phi(-t)\\
	&=(2\eta-1)\left[\phi(-t)-\phi(t)\right]
	>0
\end{align*}
since $\eta>\frac12$ and $\phi(-t)>\phi(t)$ for any $\gamma< t \leq 1$.

Consider Part~\ref{part7_lemma:quasiconcave_even}, For $\eta\in [0,\frac12)$ and any $\gamma< t \leq 1$,
\begin{align*}
	\bar{\cC}_{\phi}(t,\eta)-\bar{\cC}_{\phi}(-t,\eta)&=\eta\phi(t)+(1-\eta)\phi(-t)-\eta\phi(-t)-(1-\eta)\phi(t)\\
	&=(1-2\eta)\left[\phi(-t)-\phi(t)\right]
	>0
\end{align*}
since $\eta<\frac12$ and $\phi(-t)>\phi(t)$ for any $\gamma< t \leq 1$.

Consider Part~\ref{part8_lemma:quasiconcave_even}. For $\eta\in (\frac12,1]$ and any $0\leq t \leq 1$, 
\begin{align*}
    &\bar{\cC}_{\phi}(g(-t)-\C,\eta)-\bar{\cC}_{\phi}(g(t)+\C,\eta)\\ 
    \geq &\bar{\cC}_{\phi}(g(-t)-\C,\eta)-\bar{\cC}_{\phi}(\C-g(-t),\eta) & (g(-t)+g(t)\geq0,~\text{Part~\ref{part4_lemma:quasiconcave_even} of Lemma~\ref{lemma:quasiconcave_even}})\\
    =& (2\eta-1)[\phi(g(-t)-\C)-\phi(\C-g(-t))]\\
    >&0 & (\phi(g(-t)-\C)>\phi(\C-g(-t)))
\end{align*}

Consider Part~\ref{part9_lemma:quasiconcave_even}.
Since $\phi$ is non-increasing, for any $0\leq t \leq 1$,
\begin{align*}
    &\phi(g(-t)-\C)-\phi(\C-g(-t))+\phi(-g(t)-\C)-\phi(g(t)+\C)\\
    \geq& \phi(g(-t)-\C)-\phi(\C-g(-t))+\phi(g(t)+\C)-\phi(g(t)+\C) & (g(t)+\C>0)\\
    =&\phi(g(-t)-\C)-\phi(\C-g(-t))\\
    >&0 & (\phi(g(-t)-\C)>\phi(\C-g(-t)))
\end{align*}
$\Longleftarrow\colon$
Suppose $ \phi(\C-g(-t))+\phi(g(-t)-\C)=\phi(g(t)+\C)+\phi(-g(t)-\C)$, then for $\eta\in [0,\frac12)$,
\begin{align*}
    &\bar{\cC}_{\phi}(g(-t)-\C,\eta)-\bar{\cC}_{\phi}(g(t)+\C,\eta)\\ 
     =& (\phi(g(-t)-\C)-\phi(\C-g(-t))+\phi(-g(t)-\C)-\phi(g(t)+\C))\eta\\
     &\qquad+\phi(\C-g(-t))-\phi(-g(t)-\C)\\
     <& (\phi(g(-t)-\C)-\phi(\C-g(-t))+\phi(-g(t)-\C)-\phi(g(t)+\C))\frac12\\
     &\qquad+\phi(\C-g(-t))-\phi(-g(t)-\C)\\
    =& \frac12 (\phi(\C-g(-t))+\phi(g(-t)-\C)-\phi(g(t)+\C)-\phi(-g(t)-\C))\\
    =&0.
\end{align*}
$\Longrightarrow\colon$
Suppose $\bar{\cC}_{\phi}(g(-t)-\C,\eta)<\bar{\cC}_{\phi}(g(t)+\C,\eta)$ for $\eta\in [0,\frac12)$, then
\begin{align*}
    &\bar{\cC}_{\phi}(g(-t)-\C,\eta)-\bar{\cC}_{\phi}(g(t)+\C,\eta)\\ 
     =& (\phi(g(-t)-\C)-\phi(\C-g(-t))+\phi(-g(t)-\C)-\phi(g(t)+\C))\eta\\
     &\qquad+\phi(\C-g(-1))-\phi(-g(1)-\C)\\
     <&0 
\end{align*}
for $\eta\in [0,\frac12)$. By taking $\eta\rightarrow \frac12$, we have
\begin{align*}
     & \frac12 (\phi(\C-g(-t))+\phi(g(-t)-\C)-\phi(g(t)+\C)-\phi(-g(t)-\C))\\
     =& (\phi(g(-t)-\C)-\phi(\C-g(-t))+\phi(-g(t)-\C)-\phi(g(t)+\C))\frac12\\
     &\qquad+\phi(\C-g(-t))-\phi(-g(t)-\C)\\
     \leq&0. 
\end{align*}
By Part~\ref{part2_lemma:quasiconcave_even} of Lemma~\ref{lemma:quasiconcave_even}, we have
\begin{align*}
    &\phi(\C-g(-t))+\phi(g(-t)-\C)-\phi(g(t)+\C)-\phi(-g(t)-\C)\\
    \geq&  \phi(g(t)+\C)+\phi(-g(t)-\C)-\phi(g(t)+\C)-\phi(-g(t)-\C) &(g(-t)+g(t)\geq0)\\
    =&0.
\end{align*}
Therefore, $\phi(\C-g(-t))+\phi(g(-t)-\C)-\phi(g(t)+\C)-\phi(-g(t)-\C)=0$, i.e., $\phi(\C-g(-t))+\phi(g(-t)-\C)=\phi(g(t)+\C)+\phi(-g(t)-\C)$. 
\end{proof}

\subsection{Proof of Theorem~\ref{Thm:calibration_positive_linear} and Theorem~\ref{Thm:sup_rhomargin_calibrate_general}}
\label{app:sup_rhomargin_calibrate_general}
We will make use of general form \eqref{eq:supinf01} of the adversarial $0/1$ loss:
\begin{align*}
	\ell_{\gamma}(f,\bx,y)=\sup\limits_{\bx'\colon \|\bx-\bx'\|\leq \gamma}\mathds{1}_{y f(\bx') \leq 0}=\mathds{1}_{\inf\limits_{\bx'\colon \|\bx-\bx'\|\leq \gamma}yf(\bx')\leq 0}\,.
\end{align*}
Next, we first characterize the calibration function $\delta_{\max}(\epsilon,\bx,\eta)$ of losses $(\ell, \ell_{\gamma})$ given a symmetric hypothesis set $\sH$.
\begin{lemma}
\label{lemma:bar_delta_sup}
Let $\sH$ be a symmetric hypothesis set. For a surrogate loss $\ell$, the calibration function $\delta_{\max}(\epsilon,\bx,\eta)$ of losses $(\ell, \ell_{\gamma})$ is  
\small
\begin{align*}
\delta_{\max}(\epsilon,\bx,\eta) =
\begin{cases}
+\infty\quad & \text{if} ~ \bx \in \sX_1 \text{ or } \bx \in \sX_2, ~\epsilon>\max\curl*{\eta,1-\eta},\\
\inf\limits_{f\in\sH\colon ~\underline{M}(f,\bx,\gamma)\leq 0 \leq \overline{M}(f,\bx,\gamma)}\Delta\cC_{\ell,\sH}(f,\bx,\eta) &\text{if} ~ \bx \in \sX_2,~|2\eta-1|<\epsilon\leq\max\curl*{\eta,1-\eta},\\
\inf\limits_{f\in\sH\colon ~\underline{M}(f,\bx,\gamma)\leq 0 \leq \overline{M}(f,\bx,\gamma) \text{ or } (2\eta-1)(\underline{M}(f,\bx,\gamma))\leq 0} \Delta\cC_{\ell,\sH}(f,\bx,\eta) &\text{if} ~ \bx \in \sX_2,~\epsilon\leq|2\eta-1|,
\end{cases}
\end{align*}
\normalsize
where $\sX_1=\{\bx \in \sX: \underline{M}(f,\bx,\gamma)\leq 0 \leq \overline{M}(f,\bx,\gamma),~\forall f\in \sH\}$, $\sX_2=\{\bx \in \sX: \text{ there exists } f' \in \sH \text{ such that }\underline{M}(f',\bx,\gamma)> 0\}$ and $\sX=\sX_1 \cup \sX_2$, $\sX_1\cap \sX_2=\emptyset$.
\end{lemma}

\begin{proof}
By the definition of inner risk \eqref{eq:conditional-risk} and adversarial 0-1 loss $\ell_{\gamma}$ \eqref{eq:supinf01}, the inner $\ell_{\gamma}$-risk is 
\begin{align*}
  \cC_{\ell_{\gamma}}(f,\bx,\eta)
  &=\eta \mathds{1}_{\left\{\underline{M}(f,\bx,\gamma)\leq 0\right\}}+(1-\eta) \mathds{1}_{\left\{\overline{M}(f,\bx,\gamma)\geq 0\right\}}\\
  &=\begin{cases}
   1 & \text{if} ~ \underline{M}(f,\bx,\gamma)\leq 0 \leq \overline{M}(f,\bx,\gamma),\\
   \eta & \text{if} ~ \overline{M}(f,\bx,\gamma)<0,\\
   1-\eta & \text{if} ~ \underline{M}(f,\bx,\gamma)> 0.\\
  \end{cases}
\end{align*}
Let $\sX_1=\{\bx \in \sX: \underline{M}(f,\bx,\gamma)\leq 0 \leq \overline{M}(f,\bx,\gamma),~\forall f\in \sH\}$, $\sX_2=\{\bx \in \sX: \text{ there exists } f' \in \sH \text{ such that }\underline{M}(f',\bx,\gamma)> 0\}$. It is obvious that $\sX_1\cap \sX_2=\emptyset$.
Since $\sH$ is symmetric, for any $\bx \in \sX$, either there exists $f'\in \sH$ such that $\underline{M}(f',\bx,\gamma)>0$ and $\overline{M}(-f',\bx,\gamma)<0$, or $\underline{M}(f,\bx,\gamma)\leq 0 \leq \overline{M}(f,\bx,\gamma)$ for any $f\in \sH$. Thus $\sX=\sX_1 \cup \sX_2$. Note when $\bx \in \sX_1$, $\{f\in \sH:\overline{M}(f,\bx,\gamma)<0\}$ and $\{f\in \sH:\underline{M}(f,\bx,\gamma)>0\}$ are both empty sets. Therefore, the minimal inner $\ell_{\gamma}$-risk is
\begin{align*}
     \cC^*_{\ell_{\gamma},\sH}(\bx,\eta)=\begin{cases}
        1, & \bx \in \sX_1\,,\\
     	\min\curl*{\eta,1-\eta}, & \bx \in \sX_2\,.
     \end{cases}
\end{align*}
% Since $0\in \sH$, for any $\bx\in \sX$, $\{f\in \sH:\underline{M}(f,\bx,\gamma)\leq 0 \leq \overline{M}(f,\bx,\gamma)\}$ is not empty set.
Note when $\bx \in \sX_1$, $\cC_{\ell_{\gamma}}(f,\bx,\eta)=1$ for any $f\in \sH$, thus $\Delta\cC_{\ell_{\gamma},
    \sH}(f,\bx,\eta)=0$. When $\bx \in \sX_2$, for $f\in\{f\in \sH:\underline{M}(f,\bx,\gamma)\leq 0 \leq \overline{M}(f,\bx,\gamma)\}$, $\Delta\cC_{\ell_{\gamma},\sH}(f,\bx,\eta)=1-\min\curl*{\eta,1-\eta}=\max\curl*{\eta,1-\eta}$; for $f\in\{f\in \sH:\overline{M}(f,\bx,\gamma)<0\}$, $\Delta\cC_{\ell_{\gamma},\sH}(f,\bx,\eta)=\eta-\min\curl*{\eta,1-\eta}=\max\curl*{0,2\eta-1}=|2\eta-1|\mathds{1}_{(2\eta-1)(\underline{M}(f,\bx,\gamma))\leq 0}$ since $\underline{M}(f,\bx,\gamma)\leq \overline{M}(f,\bx,\gamma)<0$; for $f\in\{f\in \sH:\underline{M}(f,\bx,\gamma)>0\}$, $\Delta\cC_{\ell_{\gamma},\sH}(f,\bx,\eta)=1-\eta-\min\curl*{\eta,1-\eta}=\max\curl*{0,1-2\eta}=|2\eta-1|\mathds{1}_{(2\eta-1)(\underline{M}(f,\bx,\gamma))\leq 0}$ since $\underline{M}(f,\bx,\gamma)>0$. Therefore,
\begin{align}
\label{eq:Delta}
  \Delta\cC_{\ell_{\gamma},
    \sH}(f,\bx,\eta)=
    \begin{cases}
    \max\curl*{\eta,1-\eta} & \text{if} ~ \bx \in \sX_2,~\underline{M}(f,\bx,\gamma)\leq 0 \leq \overline{M}(f,\bx,\gamma),\\
   |2\eta-1|\mathds{1}_{(2\eta-1)(\underline{M}(f,\bx,\gamma))\leq 0} & \text{if} ~ \bx \in \sX_2,~\underline{M}(f,\bx,\gamma)>0 \text{~or~} \overline{M}(f,\bx,\gamma)<0,\\
   0 & \text{if} ~ \bx \in \sX_1.
    \end{cases}
\end{align}
By \eqref{eq:def-calibration-function-true}, for a fixed $\eta\in[0,1]$ and $\bx \in \sX$, the calibration function of losses $(\ell, \ell_{\gamma})$ is
\begin{align*}
    \delta_{\max}(\epsilon,\bx,\eta)=\inf_{f\in\sH}\curl*{\Delta\cC_{\ell,\sH}(f,\bx,\eta) \mid \Delta\cC_{\ell_{\gamma},
    \sH}(f,\bx,\eta)\geq\epsilon }
\end{align*}
If $\bx \in \sX_1$, then for all $f\in\sH$, $\Delta\cC_{\ell_{\gamma},\sH}(f,\bx,\eta)=0<\epsilon$, which implies that $\delta_{\max}(\epsilon,\bx,\eta)=\infty$. Next we consider case where $\bx \in \sX_2$. By the observation \eqref{eq:observation},
if $\epsilon>\max\curl*{\eta,1-\eta}$, then for all $f\in\sH$, $\Delta\cC_{\ell_{\gamma},\sH}(f,\bx,\eta)<\epsilon$, which implies that $\delta_{\max}(\epsilon,\bx,\eta)=\infty$;
if $|2\eta-1|<\epsilon\leq\max\curl*{\eta,1-\eta}$, then $\Delta\cC_{\ell_{\gamma},\sH}(f,\bx,\eta)\geq\epsilon$ if and only if $\underline{M}(f,\bx,\gamma)\leq 0 \leq \overline{M}(f,\bx,\gamma)$, which leads to
\begin{align*}
\delta_{\max}(\epsilon,\bx,\eta)=\inf_{f\in\sH\colon ~\underline{M}(f,\bx,\gamma)\leq 0 \leq \overline{M}(f,\bx,\gamma)}\Delta\cC_{\ell,\sH}(f,\bx,\eta);    
\end{align*}
if $\epsilon\leq|2\eta-1|$, then $\Delta\cC_{\ell_{\gamma},\sH}(f,\bx,\eta)\geq\epsilon$ if and only if $\underline{M}(f,\bx,\gamma)\leq 0 \leq \overline{M}(f,\bx,\gamma)$ or $(2\eta-1)(\underline{M}(f,\bx,\gamma))\leq 0$, which leads to
\begin{align*}
\delta_{\max}(\epsilon,\bx,\eta)=\inf_{f\in\sH\colon ~\underline{M}(f,\bx,\gamma)\leq 0 \leq \overline{M}(f,\bx,\gamma) \text{ or } (2\eta-1)(\underline{M}(f,\bx,\gamma))\leq 0} \Delta\cC_{\ell,\sH}(f,\bx,\eta).    
\end{align*}
\end{proof}
We then give the equivalent conditions of calibration based on inner $\ell$-risk and $\sH$.
\begin{lemma}
\label{lemma:equivalent1_sup}
Let $\sH$ be a symmetric hypothesis set and $\ell$ be a surrogate loss function. If $\sX_2=\emptyset$, any loss $\ell$ is $\sH$-calibrated with respect to $\ell_{\gamma}$. If $\sX_2\neq\emptyset$, then $\ell$ is $\sH$-calibrated with respect to $\ell_{\gamma}$ if and only if for any $\bx\in \sX_2$,
\begin{align*}
    \inf_{f\in\sH\colon ~\underline{M}(f,\bx,\gamma)\leq 0 \leq \overline{M}(f,\bx,\gamma)}\cC_{\ell}(f,\bx,\frac12)> &\inf_{f\in\sH}\cC_{\ell}(f,\bx,\frac12)\,,\text{and}\\
    \inf_{f\in\sH\colon ~\underline{M}(f,\bx,\gamma)\leq0}\cC_{\ell}(f,\bx,\eta) > &\inf_{f\in\sH}\cC_{\ell}(f,\bx,\eta) \text{ for all } \eta\in (\frac12,1]\,,\text{and}\\
    \inf_{f\in\sH\colon ~\overline{M}(f,\bx,\gamma)\geq0}\cC_{\ell}(f,\bx,\eta) > &\inf_{f\in\sH}\cC_{\ell}(f,\bx,\eta) \text{ for all } \eta\in [0,\frac12)\,.
\end{align*}
where $\sX_2=\{\bx \in \sX: \text{ there exists } f' \in \sH \text{ such that }\underline{M}(f',\bx,\gamma)> 0\}$.
\end{lemma}

\begin{proof}
Let $\delta_{\max}$ be the calibration function of $(\ell,\ell_{\gamma})$ given hypothesis set $\sH$. By Lemma~\ref{lemma:bar_delta_sup},
\small
\begin{align*}
\delta_{\max}(\epsilon,\bx,\eta) =
\begin{cases}
+\infty\quad & \text{if} ~ \bx \in \sX_1 \text{ or } \bx \in \sX_2, ~\epsilon>\max\curl*{\eta,1-\eta},\\
\inf\limits_{f\in\sH\colon ~\underline{M}(f,\bx,\gamma)\leq 0 \leq \overline{M}(f,\bx,\gamma)}\Delta\cC_{\ell,\sH}(f,\bx,\eta) &\text{if} ~ \bx \in \sX_2,~|2\eta-1|<\epsilon\leq\max\curl*{\eta,1-\eta},\\
\inf\limits_{f\in\sH\colon ~\underline{M}(f,\bx,\gamma)\leq 0 \leq \overline{M}(f,\bx,\gamma) \text{ or } (2\eta-1)(\underline{M}(f,\bx,\gamma))\leq 0} \Delta\cC_{\ell,\sH}(f,\bx,\eta) &\text{if} ~ \bx \in \sX_2,~\epsilon\leq|2\eta-1|,
\end{cases}
\end{align*}
\normalsize
where $\sX_1=\{\bx \in \sX: \underline{M}(f,\bx,\gamma)\leq 0 \leq \overline{M}(f,\bx,\gamma),~\forall f\in \sH\}$, $\sX_2=\{\bx \in \sX: \text{ there exists } f' \in \sH \text{ such that }\underline{M}(f',\bx,\gamma)> 0\}$ and $\sX=\sX_1 \cup \sX_2$, $\sX_1\cap \sX_2=\emptyset$.
By Proposition~\ref{prop:calibration_function_positive}, $\ell$ is $\sH$-calibrated with respect to $\ell_{\gamma}$ if and only if its calibration function $\delta_{\max}$ satisfies $\delta_{\max}(\epsilon,\bx,\eta)>0$ for all $\bx\in \sX$, $\eta\in [0,1]$ and $\epsilon>0$. Since $\delta(\epsilon,\bx,\eta)=\infty>0$ when $\bx \nin \sX_2$, any loss $\ell$ is $\sH$-calibrated with respect to $\ell_{\gamma}$ when $\sX_2=\emptyset$. Furtheremore, when $\sX_2\neq\emptyset$, we only need to analyze $\delta(\epsilon,\bx,\eta)$ when $\bx \in \sX_2$.\\
For $\eta=\frac12$, we have for any $\bx\in \sX_2$,
\begin{align}
\label{eq:keycondition1_sup}
\delta_{\max}(\epsilon,\bx,\frac12)>0 \text{ for all } \epsilon>0 \Leftrightarrow \inf\limits_{f\in\sH\colon ~\underline{M}(f,\bx,\gamma)\leq 0 \leq \overline{M}(f,\bx,\gamma)}\cC_{\ell}(f,\bx,\frac12)> \inf\limits_{f\in\sH}\cC_{\ell}(f,\bx,\frac12).
\end{align}
For $1\geq\eta>\frac12$, we have $|2\eta-1|=2\eta-1$, $\max\curl*{\eta,1-\eta}=\eta$, and
\begin{align*}
   \inf\limits_{f\in\sH\colon ~\underline{M}(f,\bx,\gamma)\leq 0 \leq \overline{M}(f,\bx,\gamma) \text{ or } (2\eta-1)(\underline{M}(f,\bx,\gamma))\leq 0} \Delta\cC_{\ell,\sH}(f,\bx,\eta)=
   \inf\limits_{f\in\sH\colon ~\underline{M}(f,\bx,\gamma)\leq 0} \Delta\cC_{\ell,\sH}(f,\bx,\eta)\,. 
\end{align*}
Therefore, $\delta_{\max}(\epsilon,\bx,\frac12)>0 \text{ for all } \bx\in \sX_2, \epsilon>0 \text{ and } \eta\in(\frac12,1]$ if and only if for all $\bx \in \sX_2$,
\small
\begin{align*}
\begin{cases}
\inf\limits_{f\in\sH\colon ~\underline{M}(f,\bx,\gamma)\leq 0 \leq \overline{M}(f,\bx,\gamma)}\cC_{\ell}(f,\bx,\eta)> \inf\limits_{f\in\sH}\cC_{\ell}(f,\bx,\eta) &\text{ for all } \eta\in(\frac12,1] \text{ such that } 2\eta-1<\epsilon\leq \eta,\\
\inf\limits_{f\in\sH\colon ~\underline{M}(f,\bx,\gamma)\leq0}\cC_{\ell}(f,\bx,\eta) > \inf\limits_{f\in\sH}\cC_{\ell}(f,\bx,\eta) &\text{ for all } \eta\in(\frac12,1] \text{ such that } \epsilon\leq 2\eta-1,
\end{cases}
\end{align*}
\normalsize
for all $\epsilon>0$, which is equivalent to for all $\bx \in \sX_2$,
\small
\begin{align}
\begin{cases}
\inf\limits_{f\in\sH\colon ~\underline{M}(f,\bx,\gamma)\leq 0 \leq \overline{M}(f,\bx,\gamma)}\cC_{\ell}(f,\bx,\eta)> \inf\limits_{f\in\sH}\cC_{\ell}(f,\bx,\eta) &\text{ for all } \eta\in(\frac12,1] \text{ such that } \epsilon\leq \eta < \frac{\epsilon+1}{2},\\
\inf\limits_{f\in\sH\colon ~\underline{M}(f,\bx,\gamma)\leq0}\cC_{\ell}(f,\bx,\eta) > \inf\limits_{f\in\sH}\cC_{\ell}(f,\bx,\eta) &\text{ for all } \eta\in(\frac12,1] \text{ such that } \frac{\epsilon+1}{2}\leq \eta,
\end{cases}
\label{eq:condition1 in proof_sup}
\end{align}
\normalsize
for all $\epsilon>0$.
Observe that
\begin{align*}
\begin{aligned}
    &\left\{\eta\in(\frac12,1]\Bigg|\epsilon\leq \eta < \frac{\epsilon+1}{2},\epsilon>0\right\}=\left\{\frac12<\eta\leq1\right\}\,, \text{ and}\\
    &\left\{\eta\in(\frac12,1]\Bigg|\frac{\epsilon+1}{2}\leq \eta, \epsilon>0\right\}=\left\{\frac12<\eta\leq1\right\}\,, \text{ and}\\
    &\inf\limits_{f\in\sH\colon ~\underline{M}(f,\bx,\gamma)\leq 0 \leq \overline{M}(f,\bx,\gamma)}\cC_{\ell}(f,\bx,\eta) \geq \inf\limits_{f\in\sH\colon~\underline{M}(f,\bx,\gamma)\leq0}\cC_{\ell}(f,\bx,\eta) \text{ for all } \eta\,.
\end{aligned}
\end{align*}
Therefore, we reduce the above condition~\eqref{eq:condition1 in proof_sup} as for all $\bx \in \sX_2$,
\begin{align}
    \inf\limits_{f\in\sH\colon ~\underline{M}(f,\bx,\gamma)\leq0}\cC_{\ell}(f,\bx,\eta) > \inf\limits_{f\in\sH}\cC_{\ell}(f,\bx,\eta) \text{ for all } \eta\in (\frac12,1]\,.
    \label{eq:keycondition2_sup}
\end{align}
For $\frac12>\eta\geq0$, we have $|2\eta-1|=1-2\eta$, $\max\curl*{\eta,1-\eta}=1-\eta$, and
\begin{align*}
   \inf\limits_{f\in\sH\colon ~\underline{M}(f,\bx,\gamma)\leq 0 \leq \overline{M}(f,\bx,\gamma) \text{ or } (2\eta-1)(\underline{M}(f,\bx,\gamma))\leq 0} \Delta\cC_{\ell,\sH}(f,\bx,\eta)=
   \inf\limits_{f\in\sH\colon ~\overline{M}(f,\bx,\gamma)\geq 0} \Delta\cC_{\ell,\sH}(f,\bx,\eta)\,. 
\end{align*}
Therefore, $\delta_{\max}(\epsilon,\bx,\frac12)>0 \text{ for all } \bx\in \sX_2, \epsilon>0 \text{ and } \eta\in[0,\frac12)$ if and only if for all $\bx\in \sX_2$,
\small
\begin{align*}
\begin{cases}
\inf\limits_{f\in\sH\colon ~\underline{M}(f,\bx,\gamma)\leq 0 \leq \overline{M}(f,\bx,\gamma)}\cC_{\ell}(f,\bx,\eta)> \inf\limits_{f\in\sH}\cC_{\ell}(f,\bx,\eta) &\text{ for all } \eta\in[0,\frac12) \text{ such that } 1-2\eta<\epsilon\leq 1-\eta,\\
\inf\limits_{f\in\sH\colon ~\overline{M}(f,\bx,\gamma)\geq0}\cC_{\ell}(f,\bx,\eta) > \inf\limits_{f\in\sH}\cC_{\ell}(f,\bx,\eta) &\text{ for all } \eta\in[0,\frac12) \text{ such that } \epsilon\leq 1-2\eta,
\end{cases}
\end{align*}
\normalsize
for all $\epsilon>0$, which is equivalent to for all $\bx\in \sX_2$,
\small
\begin{align}
\begin{cases}
\inf\limits_{f\in\sH\colon ~\underline{M}(f,\bx,\gamma)\leq 0 \leq \overline{M}(f,\bx,\gamma)}\cC_{\ell}(f,\bx,\eta)> \inf\limits_{f\in\sH}\cC_{\ell}(f,\bx,\eta) &\text{ for all } \eta\in[0,\frac12) \text{ such that } \frac{1-\epsilon}{2}< \eta \leq 1-\epsilon,\\
\inf\limits_{f\in\sH\colon ~\overline{M}(f,\bx,\gamma)\geq0}\cC_{\ell}(f,\bx,\eta) > \inf\limits_{f\in\sH}\cC_{\ell}(f,\bx,\eta) &\text{ for all } \eta\in[0,\frac12) \text{ such that } \eta\leq \frac{1-\epsilon}{2},
\end{cases}
\label{eq:condition2 in proof_sup}
\end{align}
\normalsize
for all $\epsilon>0$.
Observe that
\begin{align*}
    &\left\{\eta\in[0,\frac12)\Bigg|\frac{1-\epsilon}{2}< \eta \leq 1-\epsilon,\epsilon>0\right\}=\left\{0\leq\eta<\frac12\right\}\,, \text{ and}\\
    &\left\{\eta\in[0,\frac12)\Bigg|\eta\leq \frac{1-\epsilon}{2}, \epsilon>0\right\}=\left\{0\leq\eta<\frac12\right\}\,, \text{ and}\\
    &\inf\limits_{f\in\sH\colon ~\underline{M}(f,\bx,\gamma)\leq 0 \leq \overline{M}(f,\bx,\gamma)}\cC_{\ell}(f,\bx,\eta) \geq \inf\limits_{f\in\sH\colon ~\overline{M}(f,\bx,\gamma)\geq0}\cC_{\ell}(f,\bx,\eta)  \text{ for all } \eta\,.
\end{align*}
Therefore, we reduce the above condition~\eqref{eq:condition2 in proof_sup} as for all $\bx\in \sX_2$,
\begin{align}
   \inf\limits_{f\in\sH\colon ~\overline{M}(f,\bx,\gamma)\geq0}\cC_{\ell}(f,\bx,\eta) > \inf\limits_{f\in\sH}\cC_{\ell}(f,\bx,\eta) \text{ for all } \eta\in [0,\frac12)\,.
    \label{eq:keycondition3_sup}
\end{align}
To sum up, by \eqref{eq:keycondition1_sup}, \eqref{eq:keycondition2_sup} and \eqref{eq:keycondition3_sup}, we conclude the proof.
\end{proof}

Since $\sH_{\mathrm{lin}}$ is a symmetric hypothesis set, we could make use of Lemma~\ref{lemma:bar_delta_sup} and Lemma~\ref{lemma:equivalent1_sup} for proving Theorem~\ref{Thm:calibration_positive_linear}.
\CalibrationLinearPositive*

\begin{proof}
As shown by \citet{awasthi2020adversarial}, for $f\in\sH_{\mathrm{lin}}=\curl*{\bx\rightarrow \bw \cdot \bx \mid \|\bw\|=1}$,
\begin{align*}
&\underline{M}(f,\bx,\gamma)=\inf_{\bx'\colon \|\bx - \bx'\|\leq\gamma} f(\bx')=\inf_{\bx'\colon \|\bx-\bx'\|\leq \gamma}(\bw \cdot \bx')=\bw \cdot \bx-\gamma \|\bw\|=f(\bx)-\gamma, \\   
&\overline{M}(f,\bx,\gamma)=-\inf_{\bx'\colon \|\bx - \bx'\|\leq\gamma} -f(\bx')=-\inf_{\bx'\colon \|\bx-\bx'\|\leq \gamma}(-\bw \cdot \bx')=\bw \cdot \bx+\gamma \|\bw\|=f(\bx)+\gamma.
\end{align*}
Thus for $\sH_{\mathrm{lin}}$, $\sX_2=\{\bx \in \sX: \text{ there exists } f' \in \sH_{\mathrm{lin}} \text{ such that }\underline{M}(f',\bx,\gamma)> 0\}=\{\bx \in \sX: \text{ there exists } f' \in \sH_{\mathrm{lin}} \text{ such that }f'(\bx)> \gamma\}=\{\bx:\gamma<\|\bx\|\leq 1\}$ since $f(\bx)=\bw\cdot \bx\in [-\|\bx\|,\|\bx\|]$ when $f\in \sH_{\mathrm{lin}}$.
Note $\sH_{\mathrm{lin}}$ is a symmetric hypothesis set.
Therefore, by Lemma~\ref{lemma:equivalent1_sup}, $\phi$ is $\sH_{\mathrm{lin}}$-calibrated with respect to $\ell_{\gamma}$ if and only if for any $\bx\in \sX$ such that $\gamma<\|\bx\|\leq 1$,
\begin{equation}
	\begin{aligned}
	\label{eq:linear_positive_proof1}
    \inf\limits_{f\in \sH_{\mathrm{lin}}:~|f(\bx)|\leq\gamma}\cC_{\phi}(f,\bx,\frac12)> &\inf\limits_{f\in \sH_{\mathrm{lin}}}\cC_{\phi}(f,\bx,\frac12)\,,\text{and}\\
     \inf\limits_{f\in \sH_{\mathrm{lin}}:~f(\bx)\leq\gamma}\cC_{\phi}(f,\bx,\eta) > &\inf\limits_{f\in \sH_{\mathrm{lin}}}\cC_{\phi}(f,\bx,\eta)  \text{ for all } \eta\in (\frac12,1]\,,\text{and}\\
    \inf\limits_{f\in \sH_{\mathrm{lin}}:~f(\bx)\geq-\gamma}\cC_{\phi}(f,\bx,\eta) > &\inf\limits_{f\in\sH_{\mathrm{lin}}}\cC_{\phi}(f,\bx,\eta) \text{ for all } \eta\in [0,\frac12)\,.	
	\end{aligned}
\end{equation}
By the definition of inner risk \eqref{eq:conditional-risk}, the inner $\phi$-risk is
\begin{align*}
   \cC_{\phi}(f,\bx,\eta)=\eta \phi(f(\bx))+(1-\eta)\phi(-f(\bx))\,.
\end{align*}
Note $f(\bx)=\bw\cdot \bx\in [-\|\bx\|,\|\bx\|]$ when $f\in \sH_{\mathrm{lin}}$. Therefore, \eqref{eq:linear_positive_proof1} is equivalent to for any $\bx\in \sX$ such that $\gamma<\|\bx\|\leq 1$,
\begin{equation}
	\begin{aligned}
	\label{eq:linear_positive_proof2}
    \inf\limits_{-\gamma\leq t\leq\gamma}\bar{\cC}_{\phi}(t,\frac12)> &\inf\limits_{-\|\bx\|\leq t\leq\|\bx\|}\bar{\cC}_{\phi}(t,\frac12)\,,\text{and}\\
    \inf\limits_{-\|\bx\|\leq t\leq\gamma}\bar{\cC}_{\phi}(t,\eta) > &\inf\limits_{-\|\bx\|\leq t\leq\|\bx\|}\bar{\cC}_{\phi}(t,\eta)  \text{ for all } \eta\in (\frac12,1]\,,\text{and}\\
    \inf\limits_{-\gamma\leq t\leq\|\bx\|}\bar{\cC}_{\phi}(t,\eta) > &\inf\limits_{-\|\bx\|\leq t\leq\|\bx\|}\bar{\cC}_{\phi}(t,\eta) \text{ for all } \eta\in [0,\frac12)\,.	
	\end{aligned}
\end{equation}
Suppose that $\phi$ is $\sH_{\mathrm{lin}}$-calibrated with respect to $\ell_{\gamma}$. Since  by \text{Part~\ref{part2_lemma:quasiconcave_even} of Lemma~\ref{lemma:quasiconcave_even}},
\begin{align*}
 \inf\limits_{-\gamma\leq t\leq\gamma}\bar{\cC}_{\phi}(t,\frac12)=\bar{\cC}_{\phi}(\gamma,\frac12), \quad\inf\limits_{-\|\bx\|\leq t\leq\|\bx\|}\bar{\cC}_{\phi}(t,\frac12)=\bar{\cC}_{\phi}(\|\bx\|,\frac12)\,,   
\end{align*}
we obtain $\phi(\gamma)+\phi(-\gamma)=2\bar{\cC}_{\phi}(\gamma,\frac12)>2\bar{\cC}_{\phi}(t,\frac12)=\phi(t)+\phi(-t)$ for any $\gamma<t\leq1$.

Now for the other direction, assume that $\phi(\gamma)+\phi(-\gamma)>\phi(t)+\phi(-t)$ for any $\gamma<t\leq1$. For $\eta=\frac12$, by \text{Part~\ref{part2_lemma:quasiconcave_even} of Lemma~\ref{lemma:quasiconcave_even}}, we obtain for any $\bx\in \sX$ such that $\gamma<\|\bx\|\leq 1$,
\begin{align*}
\inf\limits_{-\gamma\leq t\leq\gamma}\bar{\cC}_{\phi}(t,\frac12)=\bar{\cC}_{\phi}(\gamma,\frac12)>\bar{\cC}_{\phi}(\|\bx\|,\frac12)=\inf\limits_{-\|\bx\|\leq t\leq\|\bx\|}\bar{\cC}_{\phi}(t,\frac12)\,.	
\end{align*}

For $\eta\in(\frac12,1]$ and any $\bx\in \sX$ such that $\gamma<\|\bx\|\leq 1$,
\begin{align*}
&\inf\limits_{-\|\bx\|\leq t\leq\gamma}\bar{\cC}_{\phi}(t,\eta)=\min\curl*{\bar{\cC}_{\phi}(\gamma,\eta),\bar{\cC}_{\phi}(-\|\bx\|,\eta)} \quad \text{(Part~\ref{part3_lemma:quasiconcave_even} of Lemma~\ref{lemma:quasiconcave_even})}\\
&\inf\limits_{-\|\bx\|\leq t\leq\|\bx\|}\bar{\cC}_{\phi}(t,\eta)=\min\curl*{\bar{\cC}_{\phi}(\|\bx\|,\eta),\bar{\cC}_{\phi}(-\|\bx\|,\eta)}	 \quad \text{(Part~\ref{part3_lemma:quasiconcave_even} of Lemma~\ref{lemma:quasiconcave_even})}\\
&=\bar{\cC}_{\phi}(\|\bx\|,\eta) \quad \text{(Part~\ref{part6_lemma:quasiconcave_even} of Lemma~\ref{lemma:quasiconcave_even})}
\end{align*}
Note for $\eta\in(\frac12,1]$ and any $\bx\in \sX$ such that $\gamma<\|\bx\|\leq 1$, since $\phi$ is non-increasing,
\begin{align*}
\phi(\gamma)-\phi (-\gamma)-\phi(\|\bx\|)+\phi (-\|\bx\|)\geq \phi(\|\bx\|)-\phi (-\|\bx\|)-\phi(\|\bx\|)+\phi (-\|\bx\|) =0.
\end{align*}
Thus
\begin{align*}
	\bar{\cC}_{\phi}(\gamma,\eta)-\bar{\cC}_{\phi}(\|\bx\|,\eta)&=\eta \phi(\gamma)+ (1-\eta) \phi (-\gamma) - \eta \phi(\|\bx\|)-(1-\eta) \phi (-\|\bx\|)\\
	&=\left(\phi(\gamma)-\phi (-\gamma)-\phi(\|\bx\|)+\phi (-\|\bx\|)\right)\eta+\phi (-\gamma)-\phi (-\|\bx\|)\\
	&\geq \left(\phi(\gamma)-\phi (-\gamma)-\phi(\|\bx\|)+\phi (-\|\bx\|)\right)\frac12+\phi (-\gamma)-\phi (-\|\bx\|)\\
	&=\frac12 \left[\phi(\gamma) +\phi(-\gamma)-\phi(\|\bx\|)-\phi(-\|\bx\|)\right]\\
	&> 0\,.
\end{align*}
In addition, we have for $\eta\in(\frac12,1]$ and any $\bx\in \sX$ such that $\gamma<\|\bx\|\leq 1$,
\begin{align*}
\bar{\cC}_{\phi}(-\|\bx\|,\eta)>\bar{\cC}_{\phi}(\|\bx\|,\eta).	\quad \text{(Part~\ref{part6_lemma:quasiconcave_even} of Lemma~\ref{lemma:quasiconcave_even})}
\end{align*}
Therefore for $\eta\in(\frac12,1]$ and any $\bx\in \sX$ such that $\gamma<\|\bx\|\leq 1$,
\begin{align*}
	\inf\limits_{-\|\bx\|\leq t\leq\gamma}\bar{\cC}_{\phi}(t,\eta)=\min\curl*{\bar{\cC}_{\phi}(\gamma,\eta),\bar{\cC}_{\phi}(-\|\bx\|,\eta)} > \bar{\cC}_{\phi}(\|\bx\|,\eta)=\inf\limits_{-\|\bx\|\leq t\leq\|\bx\|}\bar{\cC}_{\phi}(t,\eta)\,.
\end{align*}
For $\eta\in[0,\frac12)$ and any $\bx\in \sX$ such that $\gamma<\|\bx\|\leq 1$,
\begin{align*}
&\inf\limits_{-\gamma\leq t\leq\|\bx\|}\bar{\cC}_{\phi}(t,\eta)=\min\curl*{\bar{\cC}_{\phi}(-\gamma,\eta),\bar{\cC}_{\phi}(\|\bx\|,\eta)} \quad \text{(Part~\ref{part3_lemma:quasiconcave_even} of Lemma~\ref{lemma:quasiconcave_even})}\\
&\inf\limits_{-\|\bx\|\leq t\leq\|\bx\|}\bar{\cC}_{\phi}(t,\eta)=\min\curl*{\bar{\cC}_{\phi}(\|\bx\|,\eta),\bar{\cC}_{\phi}(-\|\bx\|,\eta)}	 \quad \text{(Part~\ref{part3_lemma:quasiconcave_even} of Lemma~\ref{lemma:quasiconcave_even})}\\
&=\bar{\cC}_{\phi}(-\|\bx\|,\eta) \quad \text{(Part~\ref{part7_lemma:quasiconcave_even} of Lemma~\ref{lemma:quasiconcave_even})}
\end{align*}
Note for $\eta\in[0,\frac12)$ and any $\bx\in \sX$ such that $\gamma<\|\bx\|\leq 1$, since $\phi$ is non-increasing,
\begin{align*}
\phi(-\gamma)-\phi (\gamma)-\phi(-\|\bx\|)+\phi (\|\bx\|)\leq \phi(-\|\bx\|)-\phi (\|\bx\|)-\phi(-\|\bx\|)+\phi (\|\bx\|) =0\,.
\end{align*}
Thus
\begin{align*}
	\bar{\cC}_{\phi}(-\gamma,\eta)-\bar{\cC}_{\phi}(-\|\bx\|,\eta)&=\eta \phi(-\gamma)+ (1-\eta) \phi (\gamma) - \eta \phi(-\|\bx\|)-(1-\eta) \phi (\|\bx\|)\\
	&=\left(\phi(-\gamma)-\phi (\gamma)-\phi(-\|\bx\|)+\phi (\|\bx\|)\right)\eta+\phi (\gamma)-\phi (\|\bx\|)\\
	&\geq \left(\phi(-\gamma)-\phi (\gamma)-\phi(-\|\bx\|)+\phi (\|\bx\|)\right)\frac12+\phi (\gamma)-\phi (\|\bx\|)\\
	&=\frac12 \left[\phi(\gamma) +\phi(-\gamma)-\phi(\|\bx\|)-\phi(-\|\bx\|)\right]\\
	&> 0\,.
\end{align*}
In addition, we have for $\eta\in[0,\frac12)$ and any $\bx\in \sX$ such that $\gamma<\|\bx\|\leq 1$,
\begin{align*}
\bar{\cC}_{\phi}(\|\bx\|,\eta)>\bar{\cC}_{\phi}(-\|\bx\|,\eta).	\quad \text{(Part~\ref{part7_lemma:quasiconcave_even} of Lemma~\ref{lemma:quasiconcave_even})}
\end{align*}
Therefore for $\eta\in[0,\frac12)$ and any $\bx\in \sX$ such that $\gamma<\|\bx\|\leq 1$,
\begin{align*}
	\inf\limits_{-\gamma\leq t\leq\|\bx\|}\bar{\cC}_{\phi}(t,\eta)=\min\curl*{\bar{\cC}_{\phi}(-\gamma,\eta),\bar{\cC}_{\phi}(\|\bx\|,\eta)} > \bar{\cC}_{\phi}(-\|\bx\|,\eta)=\inf\limits_{-\|\bx\|\leq t\leq\|\bx\|}\bar{\cC}_{\phi}(t,\eta)\,.
\end{align*}
\end{proof}

\SupRhomarginCalibrateGeneral*

\begin{proof}
By Lemma~\ref{lemma:equivalent1_sup}, if $\sX_2=\emptyset$, $\tilde{\phi}_{\rho}$ is $\sH$-calibrated with respect to $\ell_{\gamma}$. Next consider the case where $\sX_2\neq\emptyset$. By Lemma~\ref{lemma:equivalent1_sup}, $\tilde{\phi}_{\rho}$ is $\sH$-calibrated with respect to $\ell_{\gamma}$ if and only if for all $\bx\in \sX_2$,
\begin{align*}
    \inf\limits_{f\in\sH\colon ~\underline{M}(f,\bx,\gamma)\leq 0 \leq \overline{M}(f,\bx,\gamma)}\cC_{\tilde{\phi}_{\rho}}(f,\bx,\frac12)> &\inf\limits_{f\in\sH}\cC_{\tilde{\phi}_{\rho}}(f,\bx,\frac12)\,,\text{and}\\
    \inf\limits_{f\in\sH\colon ~\underline{M}(f,\bx,\gamma)\leq0}\cC_{\tilde{\phi}_{\rho}}(f,\bx,\eta) > &\inf\limits_{f\in\sH}\cC_{\tilde{\phi}_{\rho}}(f,\bx,\eta) \text{ for all } \eta\in (\frac12,1]\,,\text{and}\\
    \inf\limits_{f\in\sH\colon ~\overline{M}(f,\bx,\gamma)\geq0}\cC_{\tilde{\phi}_{\rho}}(f,\bx,\eta) > &\inf\limits_{f\in\sH}\cC_{\tilde{\phi}_{\rho}}(f,\bx,\eta) \text{ for all } \eta\in [0,\frac12)\,,
\end{align*}
where $\sX_2=\{\bx \in \sX: \text{ there exists } f' \in \sH \text{ such that }\underline{M}(f',\bx,\gamma)> 0\}$.
As shown by \citet{awasthi2020adversarial}, $\tilde{\phi}_{\rho}$ has the equivalent form
\begin{align*}
	\tilde{\phi}_{\rho}(f,\bx,y)=\phi_{\rho}\left(\inf\limits_{\bx'\colon \|\bx - \bx'\|\leq\gamma}\left(yf(\bx')\right)\right)\,.
\end{align*}
Thus by the definition of inner risk \eqref{eq:conditional-risk}, the inner $\tilde{\phi}_{\rho}$-risk is
\begin{align*}
   \cC_{\tilde{\phi}_{\rho}}(f,\bx,\eta)=\eta \phi_{\rho}(\underline{M}(f,\bx,\gamma))+(1-\eta)\phi_{\rho}(-\overline{M}(f,\bx,\gamma))\,.
\end{align*}
For any $\bx\in \sX_2$, let $M_{\bx}=\sup_{f\in \sH}\underline{M}(f,\bx,\gamma)>0$. Since $\sH$ is symmetric, we have $-M_{\bx}=\inf_{f\in \sH}\overline{M}(f,\bx,\gamma)<0$. Since $\phi_{\rho}$ is continuous, for any $\bx \in \sX_2$ and $\epsilon>0$, there exists $f_{\bx}^{\epsilon}\in \sH$ such that $\phi_{\rho}(\underline{M}(f_{\bx}^{\epsilon},\bx,\gamma))< \phi_{\rho}(M_{\bx})+\epsilon$ and $\overline{M}(f_{\bx}^{\epsilon},\bx,\gamma)\geq\underline{M}(f_{\bx}^{\epsilon},\bx,\gamma)>0$, $\underline{M}(-f_{\bx}^{\epsilon},\bx,\gamma)\leq\overline{M}(-f_{\bx}^{\epsilon},\bx,\gamma)=-\underline{M}(f_{\bx}^{\epsilon},\bx,\gamma)<0$. 
Next we analyze three cases:
\begin{itemize}
    \item 
    When $\eta=\frac12$, since $\phi_{\rho}$ is non-increasing, 
    \begin{align*}
       &\inf\limits_{f\in\sH\colon ~\underline{M}(f,\bx,\gamma)\leq 0 \leq \overline{M}(f,\bx,\gamma)}\cC_{\tilde{\phi}_{\rho}}(f,\bx,\frac12)\\
       &= \inf\limits_{f\in\sH\colon ~\underline{M}(f,\bx,\gamma)\leq 0 \leq \overline{M}(f,\bx,\gamma)}\frac12 \phi_{\rho}(\underline{M}(f,\bx,\gamma))+\frac12\phi_{\rho}(-\overline{M}(f,\bx,\gamma))\\
       &\geq \frac12 \phi_{\rho}(0)+\frac12 \phi_{\rho}(0)=\phi_{\rho}(0)=1\,.
    \end{align*}
%    Take $f=0\in \sH$, then for any $\bx\in \sX$, $\underline{M}(f,\bx,\gamma)=\overline{M}(f,\bx,\gamma)=0$, $\cC_{\tilde{\phi}_{\rho}}(f,\bx,\frac12)=\frac12 \phi_{\rho}(0)+\frac12 \phi_{\rho}(0)=\phi_{\rho}(0)=1$. Therefore for any $\bx\in \sX$,
%    \begin{align*}
%     \inf\limits_{f\in\sH\colon ~\underline{M}(f,\bx,\gamma)\leq 0 \leq \overline{M}(f,\bx,\gamma)}\cC_{\tilde{\phi}_{\rho}}(f,\bx,\frac12)=1\,.   
%    \end{align*}
    For any $\bx \in \sX_2$, there exists $f'\in \sH$ such that $\underline{M}(f',\bx,\gamma)>0$ and $-\overline{M}(f',\bx,\gamma)\leq-\underline{M}(f',\bx,\gamma)<0$, we obtain \[\cC_{\tilde{\phi}_{\rho}}(f',\bx,\frac12)=\frac12\phi_{\rho}(\underline{M}(f',\bx,\gamma))+\frac12\phi_{\rho}(-\overline{M}(f',\bx,\gamma))=\frac12\phi_{\rho}(\underline{M}(f',\bx,\gamma))+\frac12<1.\] Therefore for any $\bx \in \sX_2$,
    \begin{align}
        \inf\limits_{f\in\sH}\cC_{\tilde{\phi}_{\rho}}(f,\bx,\frac12)
        \leq 
        \cC_{\tilde{\phi}_{\rho}}(f',\bx,\frac12)
        < 1\leq\inf\limits_{f\in\sH\colon ~\underline{M}(f,\bx,\gamma)\leq 0 \leq \overline{M}(f,\bx,\gamma)}\cC_{\tilde{\phi}_{\rho}}(f,\bx,\frac12)\,.
        \label{eq:RNN1}
    \end{align}
    \item 
    When $\eta\in (\frac12,1]$, since $\phi_{\rho}$ is non-increasing, for any $\bx \in \sX_2$,
    \begin{align*}
    \inf\limits_{f\in\sH\colon ~\underline{M}(f,\bx,\gamma)\leq0}\cC_{\tilde{\phi}_{\rho}}(f,\bx,\eta)
    &=\inf\limits_{f\in\sH\colon ~\underline{M}(f,\bx,\gamma)\leq0} \eta\phi_{\rho}(\underline{M}(f,\bx,\gamma))+(1-\eta)\phi_{\rho}(-\overline{M}(f,\bx,\gamma))\\
    &= \eta + \inf\limits_{f\in\sH\colon ~\underline{M}(f,\bx,\gamma)\leq0} (1-\eta)\phi_{\rho}(-\overline{M}(f,\bx,\gamma))\\
    &\geq \eta + (1-\eta)\phi_{\rho}(M_{\bx})\,.
\end{align*}
 On the other hand, for any $\bx \in \sX_2$ and $\epsilon>0$, 
 \begin{align*}
 	\cC_{\tilde{\phi}_{\rho}}(f_{\bx}^{\epsilon},\bx,\eta)
 	&=\eta\phi_{\rho}(\underline{M}(f_{\bx}^{\epsilon},\bx,\gamma))+(1-\eta)\phi_{\rho}(-\overline{M}(f_{\bx}^{\epsilon},\bx,\gamma))\\
 	&<\eta \phi_{\rho}(M_{\bx})+\epsilon +(1-\eta)\,.
 \end{align*}
 Since $\eta>\frac12$ and $M_{\bx}>0$, we have
 \begin{align*}
     &\inf\limits_{f\in\sH\colon ~\underline{M}(f,\bx,\gamma)\leq0}\cC_{\tilde{\phi}_{\rho}}(f,\bx,\eta)-\cC_{\tilde{\phi}_{\rho}}(f_{\bx}^{\epsilon},\bx,\eta)\\
     &>[\eta + (1-\eta)\phi_{\rho}(M_{\bx})]- [\eta \phi_{\rho}(M_{\bx})+\epsilon +(1-\eta)]\\
     &=(2\eta-1)(1-\phi_{\rho}(M_{\bx}))-\epsilon\\
     &>0,
 \end{align*}
 where we take $0<\epsilon<(2\eta-1)(1-\phi_{\rho}(M_{\bx}))$.

 Therefore for any $\eta\in (\frac12,1]$ and $\bx \in \sX_2$, there exists $0<\epsilon<(2\eta-1)(1-\phi_{\rho}(M_{\bx}))$ such that
    \begin{align}
        \inf\limits_{f\in\sH}\cC_{\tilde{\phi}_{\rho}}(f,\bx,\eta)\leq \cC_{\tilde{\phi}_{\rho}}(f_{\bx}^{\epsilon},\bx,\eta)
        <\inf\limits_{f\in\sH\colon ~\underline{M}(f,\bx,\gamma)\leq 0}\cC_{\tilde{\phi}_{\rho}}(f,\bx,\eta)\,.
        \label{eq:RNN2}
    \end{align}
     \item 
    When $\eta\in [0,\frac12)$, since $\phi_{\rho}$ is non-increasing, for any $\bx \in \sX_2$,
    \begin{align*}
    \inf\limits_{f\in\sH\colon ~\overline{M}(f,\bx,\gamma)\geq0}\cC_{\tilde{\phi}_{\rho}}(f,\bx,\eta)
    &=\inf\limits_{f\in\sH\colon ~\overline{M}(f,\bx,\gamma)\geq0}\eta\phi_{\rho}(\underline{M}(f,\bx,\gamma))+(1-\eta)\phi_{\rho}(-\overline{M}(f,\bx,\gamma))\\
    &= 1-\eta + \inf\limits_{f\in\sH\colon ~\overline{M}(f,\bx,\gamma)\geq0}\eta\phi_{\rho}(\underline{M}(f,\bx,\gamma))\\
    &\geq 1-\eta +\eta \phi_{\rho}(M_{\bx})
\end{align*}
On the other hand, for any $\bx \in \sX_2$ and $\epsilon>0$, 
\begin{align*}
\cC_{\tilde{\phi}_{\rho}}(-f_{\bx}^{\epsilon},\bx,\eta)
&=\eta\phi_{\rho}(\underline{M}(-f_{\bx}^{\epsilon},\bx,\gamma))+(1-\eta)\phi_{\rho}(-\overline{M}(-f_{\bx}^{\epsilon},\bx,\gamma))\\
&=\eta +(1-\eta)\phi_{\rho}(\underline{M}(f_{\bx}^{\epsilon},\bx,\gamma))\\
&<\eta +(1-\eta)\phi_{\rho}(M_{\bx})+\epsilon
\end{align*}
Since $\eta<\frac12$ and $M_{\bx}>0$, we have
\begin{align*}
    &\inf\limits_{f\in\sH\colon ~\overline{M}(f,\bx,\gamma)\geq0}\cC_{\tilde{\phi}_{\rho}}(f,\bx,\eta)-\cC_{\tilde{\phi}_{\rho}}(-f_{\bx}^{\epsilon},\bx,\eta)\\
    &>[1-\eta + \eta\phi_{\rho}(M_{\bx})]- [\eta +(1-\eta)\phi_{\rho}(M_{\bx})+\epsilon]\\
    &=(1-2\eta)(1-\phi_{\rho}(M_{\bx}))-\epsilon\\
    &>0
\end{align*}
where we take $0<\epsilon<(1-2\eta)(1-\phi_{\rho}(M_{\bx}))$.

Therefore for any $\eta\in [0,\frac12)$ and $\bx \in \sX_2$, there exists $0<\epsilon<(1-2\eta)(1-\phi_{\rho}(M_{\bx}))$ such that
    \begin{align}
        \inf\limits_{f\in\sH}\cC_{\tilde{\phi}_{\rho}}(f,\bx,\eta)\leq
       \cC_{\tilde{\phi}_{\rho}}(-f_{\bx}^{\epsilon},\bx,\eta)<\inf\limits_{f\in\sH\colon ~ \overline{M}(f,\bx,\gamma)\geq0}\cC_{\tilde{\phi}_{\rho}}(f,\bx,\eta)\,.
        \label{eq:RNN3}
    \end{align}
\end{itemize}
To sum up, by \eqref{eq:RNN1}, \eqref{eq:RNN2} and \eqref{eq:RNN3}, we conclude that $\tilde{\phi}_{\rho}$ is $\sH$-calibrated with respect to $\ell_{\gamma}$.
\end{proof}

\subsection{Proof of Theorem~\ref{Thm:quasiconcave_calibrate_general}}
\label{app:quasiconcave_calibrate_general}
As shown by \citet{awasthi2020adversarial}, for $f\in\sH_g$, the adversarial $0/1$ loss has the equivalent form
\begin{align}
\label{eq:supinf01_general}
	\ell_{\gamma}(f,\bx,y)=\mathds{1}_{\inf\limits_{\bx'\colon \|\bx-\bx'\|\leq \gamma}\left(y g(\bw \cdot \bx')+by\right)\leq 0}=\mathds{1}_{
	y g(\bw \cdot \bx -\gamma y\|\bw\|)+by\leq 0}=\mathds{1}_{yg(\bw \cdot \bx -\gamma y)+by \leq 0}\,.
\end{align}
The proofs of Theorem~\ref{Thm:quasiconcave_calibrate_general} will closely follow the proofs of Theorem~\ref{Thm:calibration_positive_linear} and Theorem~\ref{Thm:sup_rhomargin_calibrate_general}. We will first prove Lemma~\ref{lemma:bar_delta_general} and Lemma~\ref{lemma:equivalent1_general} analogous to Lemma~\ref{lemma:bar_delta_sup} and Lemma~\ref{lemma:equivalent1_sup} respectively.
Without loss of generality, assume that $g$ is continuous and satisfies $g(-1-\gamma)+\C>0$, $g(1+\gamma)-\C<0$. Then observe that $g(-\gamma)+\C>0$, $g(\gamma)-\C<0$ since $g$ is non-decreasing.
\begin{lemma}
\label{lemma:bar_delta_general}
For a surrogate loss $\ell$ and hypothesis set $\sH_g$, the calibration function of losses $(\ell, \ell_{\gamma})$ is 
\small
\begin{align*}
\delta_{\max}(\epsilon,\bx,\eta) =
\begin{cases}
+\infty & \text{if}~ \epsilon>\max\curl*{\eta,1-\eta},\\
\inf_{f\in \sH_g:~g(\bw \cdot \bx -\gamma)+b\leq 0\leq g(\bw \cdot \bx +\gamma)+b}\Delta\cC_{\ell,\sH_g}(f,\bx,\eta) & \text{if} ~ |2\eta-1|<\epsilon\leq\max\curl*{\eta,1-\eta},\\
\inf_{f\in \sH_g:~g(\bw \cdot \bx -\gamma)+b\leq 0\leq g(\bw \cdot \bx +\gamma)+b \text{ or }(2\eta-1)[g(\bw \cdot \bx -\gamma)+b]\leq 0}\Delta\cC_{\ell,\sH_g}(f,\bx,\eta) & \text{if} ~ \epsilon\leq|2\eta-1|.
\end{cases}
\end{align*}
\end{lemma}
\normalsize
\begin{proof}
As with the proof of Lemma~\ref{lemma:bar_delta_sup}, we first characterize the inner $\ell$-risk and minimal inner $\ell_{\gamma}$-risk for $\sH_g$.
By the definition of inner risk \eqref{eq:conditional-risk} and equivalent form of adversarial 0-1 loss $\ell_{\gamma}$ for $\sH_g$ \eqref{eq:supinf01_general}, the inner $\ell_{\gamma}$-risk is
\begin{align*}
  \cC_{\ell_{\gamma}}(f,\bx,\eta)&=\eta \mathds{1}_{g(\bw \cdot \bx -\gamma)+b \leq 0}+(1-\eta)\mathds{1}_{g(\bw \cdot \bx +\gamma)+b \geq 0}\\
  &=\begin{cases}
   1 & \text{if} ~ g(\bw \cdot \bx -\gamma)+b\leq 0\leq g(\bw \cdot \bx +\gamma)+b,\\
   \eta & \text{if} ~ g(\bw \cdot \bx +\gamma)+b< 0,\\
   1-\eta & \text{if} ~ g(\bw \cdot \bx -\gamma)+b> 0.\\
  \end{cases}
\end{align*}
where we used the fact that $g$ is non-decreasing and $g(\bw \cdot \bx -\gamma)\leq g(\bw \cdot \bx +\gamma)$.
Note for any $\bx \in \sX$, $\bw \cdot \bx\in [-\|\bx\|,\|\bx\|]$. Thus we have $g(\bw \cdot \bx -\gamma)+b\in [g(-\|\bx\|-\gamma)-G,g(\|\bx\|-\gamma)+G]$ and $g(\bw \cdot \bx +\gamma)+b\in [g(-\|\bx\|+\gamma)-G,g(\|\bx\|+\gamma)+G]$ since $g$ is non-decreasing. By the fact that $g(-\gamma)+\C>0$ and $g(\gamma)-\C<0$, we obtain the minimal inner $\ell_{\gamma}$-risk, which is for any $\bx \in \sX$,
\begin{align*}
    \cC^*_{\ell_{\gamma},\sH_g}(\bx,\eta)=\min\curl*{\eta,1-\eta}\,.
\end{align*}
As with the derivation of $\Delta\cC_{\ell_{\gamma},\sH}(f,\bx,\eta)$ \eqref{eq:Delta}, we derive $\Delta\cC_{\ell_{\gamma},\sH_g}(f,\bx,\eta)$ as follows. By the observation \eqref{eq:observation}, for any $\bx \in \sX$, for $f\in \sH_g$ such that $g(\bw \cdot \bx -\gamma)+b\leq 0\leq g(\bw \cdot \bx +\gamma)+b$, $\Delta\cC_{\ell_{\gamma},\sH_g}(f,\bx,\eta)=1-\min\curl*{\eta,1-\eta}=\max\curl*{\eta,1-\eta}$; for $f\in \sH_g$ such that $g(\bw \cdot \bx +\gamma)+b<0$, $\Delta\cC_{\ell_{\gamma},\sH_g}(f,\bx,\eta)=\eta-\min\curl*{\eta,1-\eta}=\max\curl*{0,2\eta-1}=|2\eta-1|\mathds{1}_{(2\eta-1)[g(\bw \cdot \bx -\gamma)+b]\leq 0}$ since $g(\bw \cdot \bx -\gamma)+b\leq g(\bw \cdot \bx +\gamma)+b<0$; for $f\in \sH_g$ such that $g(\bw \cdot \bx -\gamma)+b>0$, $\Delta\cC_{\ell_{\gamma},\sH_g}(f,\bx,\eta)=1-\eta-\min\curl*{\eta,1-\eta}=\max\curl*{0,1-2\eta}=|2\eta-1|\mathds{1}_{(2\eta-1)[g(\bw \cdot \bx -\gamma)+b]\leq 0}$ since $g(\bw \cdot \bx -\gamma)+b>0$. Therefore,
\begin{align*}
  \Delta\cC_{\ell_{\gamma},
    \sH_g}(f,\bx,\eta)&=
    \begin{cases}
    \max\curl*{\eta,1-\eta} & \text{if} ~ g(\bw \cdot \bx -\gamma)+b\leq 0\leq g(\bw \cdot \bx +\gamma)+b\,,\\
   |2\eta-1|\mathds{1}_{(2\eta-1)[g(\bw \cdot \bx -\gamma)+b]\leq 0} & \text{if} ~ g(\bw \cdot \bx +\gamma)+b<0 \text{ or }g(\bw \cdot \bx -\gamma)+b>0\,.
    \end{cases}
\end{align*}
By \eqref{eq:def-calibration-function-true}, for a fixed $\eta\in[0,1]$ and $\bx \in \sX$, the calibration function of losses $(\ell, \ell_{\gamma})$ given $\sH_g$ is
\begin{align*}
    \delta_{\max}(\epsilon,\bx,\eta)=\inf_{f\in\sH_g}\curl*{\Delta\cC_{\ell,\sH_g}(f,\bx,\eta) \mid \Delta\cC_{\ell_{\gamma},\sH_g}(f,\bx,\eta)\geq\epsilon }.
\end{align*}
As with the proof of Lemma~\ref{lemma:bar_delta_sup}, we then make use of the observation \eqref{eq:observation} for deriving the the calibration function. By the observation \eqref{eq:observation},
if $\epsilon>\max\curl*{\eta,1-\eta}$, then for all $f\in\sH_g$, $\Delta\cC_{\ell_{\gamma},\sH_g}(f,\bx,\eta)<\epsilon$, which implies that $\delta_{\max}(\epsilon,\bx,\eta)=\infty$;
if $|2\eta-1|<\epsilon\leq\max\curl*{\eta,1-\eta}$, then $\Delta\cC_{\ell_{\gamma},\sH_g}(f,\bx,\eta)\geq\epsilon$ if and only if $g(\bw \cdot \bx -\gamma)+b\leq 0\leq g(\bw \cdot \bx +\gamma)+b$, which leads to
\begin{align*}
 \delta_{\max}(\epsilon,\bx,\eta)=\inf_{f\in \sH_g:~g(\bw \cdot \bx -\gamma)+b\leq 0\leq g(\bw \cdot \bx +\gamma)+b}\Delta\cC_{\ell,\sH_g}(f,\bx,\eta);   
\end{align*}
if $\epsilon\leq|2\eta-1|$, then $\Delta\cC_{\ell_{\gamma},\sH_g}(f,\bx,\eta)\geq\epsilon$ if and only if $g(\bw \cdot \bx -\gamma)+b\leq 0\leq g(\bw \cdot \bx +\gamma)+b \text{ or }(2\eta-1)[g(\bw \cdot \bx -\gamma)+b]\leq 0$, which leads to
\begin{align*}
\delta_{\max}(\epsilon,\bx,\eta)=\inf_{f\in \sH_g:~g(\bw \cdot \bx -\gamma)+b\leq 0\leq g(\bw \cdot \bx +\gamma)+b \text{ or }(2\eta-1)[g(\bw \cdot \bx -\gamma)+b]\leq 0}\Delta\cC_{\ell,\sH_g}(f,\bx,\eta). 
\end{align*}
\end{proof}

\begin{lemma}
\label{lemma:equivalent1_general}
Let $\ell$ be a surrogate loss function. Then $\ell$ is $\sH_g$-calibrated with respect to $\ell_{\gamma}$ if and only if for any $\bx\in \sX$,
\begin{align*}
    \inf\limits_{f\in\sH_g:~g(\bw \cdot \bx -\gamma)+b\leq 0\leq g(\bw \cdot \bx +\gamma)+b}\cC_{\ell}(f,\bx,\frac12)> &\inf\limits_{f\in \sH_g}\cC_{\ell}(f,\bx,\frac12)\,,\text{and}\\
     \inf\limits_{f\in \sH_g:~g(\bw \cdot \bx -\gamma)+b\leq 0}\cC_{\ell}(f,\bx,\eta) > &\inf\limits_{f\in \sH_g}\cC_{\ell}(f,\bx,\eta)  \text{ for all } \eta\in (\frac12,1]\,,\text{and}\\
    \inf\limits_{f\in \sH_g:~g(\bw \cdot \bx +\gamma)+b\geq 0}\cC_{\ell}(f,\bx,\eta) > &\inf\limits_{f\in\sH_g}\cC_{\ell}(f,\bx,\eta) \text{ for all } \eta\in [0,\frac12)\,.
\end{align*}
\end{lemma}
\begin{proof}
As the proof of Lemma~\ref{lemma:equivalent1_sup} first makes use of Lemma~\ref{lemma:bar_delta_sup} and Proposition~\ref{prop:calibration_function_positive}, we also first make use of Lemma~\ref{lemma:bar_delta_general} and Proposition~\ref{prop:calibration_function_positive} in the following proof.
Let $\delta_{\max}$ be the calibration function of $(\ell,\ell_{\gamma})$ for hypothesis set $\sH_g$. By Lemma~\ref{lemma:bar_delta_general},
\small
\begin{align*}
\delta_{\max}(\epsilon,\bx,\eta) =
\begin{cases}
+\infty & \text{if}~ \epsilon>\max\curl*{\eta,1-\eta},\\
\inf_{f\in \sH_g:~g(\bw \cdot \bx -\gamma)+b\leq 0\leq g(\bw \cdot \bx +\gamma)+b}\Delta\cC_{\ell,\sH_g}(f,\bx,\eta) & \text{if} ~ |2\eta-1|<\epsilon\leq\max\curl*{\eta,1-\eta},\\
\inf_{f\in \sH_g:~g(\bw \cdot \bx -\gamma)+b\leq 0\leq g(\bw \cdot \bx +\gamma)+b \text{ or }(2\eta-1)[g(\bw \cdot \bx -\gamma)+b]\leq 0}\Delta\cC_{\ell,\sH_g}(f,\bx,\eta) & \text{if} ~ \epsilon\leq|2\eta-1|.
\end{cases}
\end{align*}
\normalsize
By Proposition~\ref{prop:calibration_function_positive}, $\ell$ is $\sH_g$-calibrated with respect to $\ell_{\gamma}$ if and only if its calibration function $\delta_{\max}$ satisfies $\delta_{\max}(\epsilon,\bx,\eta)>0$ for all $\bx\in \sX$, $\eta\in [0,1]$ and $\epsilon>0$.
The following steps are similar to the steps in the proof of Lemma~\ref{lemma:equivalent1_sup}, where we analyze by considering three cases.\\
For $\eta=\frac12$, we have for any $\bx \in \sX$,
\begin{align}
\delta_{\max}(\epsilon,\bx,\frac12)>0 \text{ for all } \epsilon>0 \Leftrightarrow \inf\limits_{f\in \sH_g:~g(\bw \cdot \bx -\gamma)+b\leq 0\leq g(\bw \cdot \bx +\gamma)+b}\cC_{\ell}(f,\bx,\frac12)> \inf\limits_{f\in \sH_g}\cC_{\ell}(f,\bx,\frac12).
\label{eq:keycondition1_general}
\end{align}
For $1\geq\eta>\frac12$, we have $|2\eta-1|=2\eta-1$, $\max\curl*{\eta,1-\eta}=\eta$, and
\begin{align*}
   \inf_{f\in \sH_g:~g(\bw \cdot \bx -\gamma)+b\leq 0\leq g(\bw \cdot \bx +\gamma)+b \text{ or }(2\eta-1)[g(\bw \cdot \bx -\gamma)+b]\leq 0}\Delta\cC_{\ell,\sH_g}(f,\bx,\eta)
   =\inf_{f\in \sH_g\colon ~g(\bw \cdot \bx -\gamma)+b\leq 0} \Delta\cC_{\ell,\sH_g}(f,\bx,\eta)\,. 
\end{align*}
Therefore, $\delta_{\max}(\epsilon,\bx,\frac12)>0 \text{ for any } \bx \in \sX,~\epsilon>0 \text{ and } \eta\in(\frac12,1]$ if and only if$\text{ for any } \bx \in \sX$,
\small
\begin{align*}
\begin{cases}
\inf\limits_{f\in \sH_g:~g(\bw \cdot \bx -\gamma)+b\leq 0\leq g(\bw \cdot \bx +\gamma)+b}\cC_{\ell}(f,\bx,\eta)> \inf\limits_{f\in \sH_g}\cC_{\ell}(f,\bx,\eta) &\text{ for all } \eta\in(\frac12,1] \text{ such that } 2\eta-1<\epsilon\leq \eta,\\
\inf\limits_{f\in \sH_g\colon ~g(\bw \cdot \bx -\gamma)+b\leq 0}\cC_{\ell}(f,\bx,\eta) > \inf\limits_{f\in \sH_g}\cC_{\ell}(f,\bx,\eta) &\text{ for all } \eta\in(\frac12,1] \text{ such that } \epsilon\leq 2\eta-1,
\end{cases}
\end{align*}
\normalsize
for all $\epsilon>0$, which is equivalent to $\text{for any } \bx \in \sX$,
\small
\begin{align}
\begin{cases}
\label{eq:condition1 in proof_general}
\inf\limits_{f\in \sH_g:~g(\bw \cdot \bx -\gamma)+b\leq 0\leq g(\bw \cdot \bx +\gamma)+b}\cC_{\ell}(f,\bx,\eta)> \inf\limits_{f\in \sH_g}\cC_{\ell}(f,\bx,\eta) &\text{ for all } \eta\in(\frac12,1] \text{ such that } \epsilon\leq \eta < \frac{\epsilon+1}{2},\\
\inf\limits_{f\in \sH_g\colon ~g(\bw \cdot \bx -\gamma)+b\leq 0}\cC_{\ell}(f,\bx,\eta) > \inf\limits_{f\in \sH_g}\cC_{\ell}(f,\bx,\eta) &\text{ for all } \eta\in(\frac12,1] \text{ such that } \frac{\epsilon+1}{2}\leq \eta,
\end{cases}
\end{align}
\normalsize
for all $\epsilon>0$.
Observe that
\begin{align*}
    &\left\{\eta\in(\frac12,1]\Bigg|\epsilon\leq \eta < \frac{\epsilon+1}{2},\epsilon>0\right\}=\left\{\frac12<\eta\leq1\right\}\,, \text{ and}\\
    &\left\{\eta\in(\frac12,1]\Bigg|\frac{\epsilon+1}{2}\leq \eta, \epsilon>0\right\}=\left\{\frac12<\eta\leq1\right\}\,, \text{ and}\\
    &\inf\limits_{f\in \sH_g:~g(\bw \cdot \bx -\gamma)+b\leq 0\leq g(\bw \cdot \bx +\gamma)+b}\cC_{\ell}(f,\bx,\eta) \geq \inf\limits_{f\in \sH_g\colon ~g(\bw \cdot \bx -\gamma)+b\leq 0}\cC_{\ell}(f,\bx,\eta) \text{ for all } \eta\,.
\end{align*}
Therefore, we reduce the above condition~\eqref{eq:condition1 in proof_general} as $\text{for any } \bx \in \sX$,
\begin{align}
\label{eq:keycondition2_general}
    \inf\limits_{f\in \sH_g\colon ~g(\bw \cdot \bx -\gamma)+b\leq 0}\cC_{\ell}(f,\bx,\eta) > \inf\limits_{f\in \sH_g}\cC_{\ell}(f,\bx,\eta) \text{ for all } \eta\in (\frac12,1]\,.
\end{align}
For $\frac12>\eta\geq0$, we have $|2\eta-1|=1-2\eta$, $\max\curl*{\eta,1-\eta}=1-\eta$, and
\begin{align*}
   \inf_{f\in \sH_g:~g(\bw \cdot \bx -\gamma)+b\leq 0\leq g(\bw \cdot \bx +\gamma)+b \text{ or }(2\eta-1)[g(\bw \cdot \bx -\gamma)+b]\leq 0}\Delta\cC_{\ell,\sH_g}(f,\bx,\eta)
   =\inf_{f\in \sH_g\colon ~g(\bw \cdot \bx +\gamma)+b\geq 0} \Delta\cC_{\ell,\sH_g}(f,\bx,\eta)\,.  
\end{align*}
Therefore, $\delta_{\max}(\epsilon,\bx,\frac12)>0 \text{ for any } \bx \in \sX,~\epsilon>0 \text{ and } \eta\in[0,\frac12)$ if and only if$\text{ for any } \bx \in \sX$,
\small
\begin{align*}
\begin{cases}
\inf\limits_{f\in \sH_g:~g(\bw \cdot \bx -\gamma)+b\leq 0\leq g(\bw \cdot \bx +\gamma)+b}\cC_{\ell}(f,\bx,\eta)> \inf\limits_{f\in \sH_g}\cC_{\ell}(f,\bx,\eta) &\text{ for all } \eta\in[0,\frac12) \text{ such that } 1-2\eta<\epsilon\leq 1-\eta,\\
\inf\limits_{f\in \sH_g\colon ~g(\bw \cdot \bx +\gamma)+b\geq 0}\cC_{\ell}(f,\bx,\eta) > \inf\limits_{f\in \sH_g}\cC_{\ell}(f,\bx,\eta) &\text{ for all } \eta\in[0,\frac12) \text{ such that } \epsilon\leq 1-2\eta,
\end{cases}
\end{align*}
\normalsize
for all $\epsilon>0$, which is equivalent to $\text{for any } \bx \in \sX$,
\small
\begin{align}
\begin{cases}
\inf\limits_{f\in \sH_g:~g(\bw \cdot \bx -\gamma)+b\leq 0\leq g(\bw \cdot \bx +\gamma)+b}\cC_{\ell}(f,\bx,\eta)> \inf\limits_{f\in \sH_g}\cC_{\ell}(f,\bx,\eta) &\text{ for all } \eta\in[0,\frac12) \text{ such that } \frac{1-\epsilon}{2}< \eta \leq 1-\epsilon,\\
\inf\limits_{f\in \sH_g\colon ~g(\bw \cdot \bx +\gamma)+b\geq 0}\cC_{\ell}(f,\bx,\eta) > \inf\limits_{f\in \sH_g}\cC_{\ell}(f,\bx,\eta) &\text{ for all } \eta\in[0,\frac12) \text{ such that } \eta\leq \frac{1-\epsilon}{2},
\end{cases}
\label{eq:condition2 in proof_general}
\end{align}
\normalsize
for all $\epsilon>0$.
Observe that
\begin{align*}
    &\left\{\eta\in[0,\frac12)\Bigg|\frac{1-\epsilon}{2}< \eta \leq 1-\epsilon,\epsilon>0\right\}=\left\{0\leq\eta<\frac12\right\}\,, \text{ and}\\
    &\left\{\eta\in[0,\frac12)\Bigg|\eta\leq \frac{1-\epsilon}{2}, \epsilon>0\right\}=\left\{0\leq\eta<\frac12\right\}\,, \text{ and}\\
    &\inf\limits_{f\in \sH_g:~g(\bw \cdot \bx -\gamma)+b\leq 0\leq g(\bw \cdot \bx +\gamma)+b}\cC_{\ell}(f,\bx,\eta) \geq \inf\limits_{f\in \sH_g\colon ~g(\bw \cdot \bx +\gamma)+b\geq 0}\cC_{\ell}(f,\bx,\eta) \text{ for all } \eta\,.
\end{align*}
Therefore we reduce the above condition~\eqref{eq:condition2 in proof_general} as $\text{for any } \bx \in \sX$,
\begin{align}
    \inf\limits_{f\in \sH_g\colon ~g(\bw \cdot \bx +\gamma)+b\geq 0}\cC_{\ell}(f,\bx,\eta) > \inf\limits_{f\in \sH_g}\cC_{\ell}(f,\bx,\eta) \text{ for all } \eta\in [0,\frac12)\,.
    \label{eq:keycondition3_general}
\end{align}
To sum up, by \eqref{eq:keycondition1_general}, \eqref{eq:keycondition2_general} and \eqref{eq:keycondition3_general}, we conclude the proof.
\end{proof}

\QuasiconcaveCalibrateGeneral*
\begin{proof}
By Lemma~\ref{lemma:equivalent1_general}, $\phi$ is $\sH_g$-calibrated with respect to $\ell_{\gamma}$ if and only if for any $\bx\in \sX$,
\begin{equation}
\label{eq:general_positive_proof1}
\begin{aligned}
    \inf\limits_{f\in\sH_g:~g(\bw \cdot \bx -\gamma)+b\leq 0\leq g(\bw \cdot \bx +\gamma)+b}\cC_{\phi}(f,\bx,\frac12)> &\inf\limits_{f\in \sH_g}\cC_{\phi}(f,\bx,\frac12)\,,\text{and}\\
     \inf\limits_{f\in \sH_g:~g(\bw \cdot \bx -\gamma)+b\leq 0}\cC_{\phi}(f,\bx,\eta) > &\inf\limits_{f\in \sH_g}\cC_{\phi}(f,\bx,\eta)  \text{ for all } \eta\in (\frac12,1]\,,\text{and}\\
    \inf\limits_{f\in \sH_g:~g(\bw \cdot \bx +\gamma)+b\geq 0}\cC_{\phi}(f,\bx,\eta) > &\inf\limits_{f\in\sH_g}\cC_{\phi}(f,\bx,\eta) \text{ for all } \eta\in [0,\frac12)\,.
\end{aligned}
\end{equation}
By the definition of inner risk \eqref{eq:conditional-risk}, the inner $\phi$-risk is
\begin{align*}
   \cC_{\phi}(f,\bx,\eta)=\eta \phi(f(\bx))+(1-\eta)\phi(-f(\bx))\,.
\end{align*}
and $f(\bx)=g(\bw\cdot \bx)+b\in [g(-\|\bx\|)-\C,g(\|\bx\|)+\C]$ when $f\in \sH_g$ since $g$ is continuous and non-decreasing. Specifically, by the assumption that $g(-1-\gamma)+\C>0$, $g(1+\gamma)-\C<0$, when $f\in\{f\in\sH_g:~g(\bw \cdot \bx -\gamma)+b\leq 0\leq g(\bw \cdot \bx +\gamma)+b\}$, $f(\bx)=g(\bw\cdot \bx)+b\in [\min_{-\|\bx\|\leq s\leq \|\bx\|}g(s)-g(s+\gamma),\max_{-\|\bx\|\leq s\leq \|\bx\|}g(s)-g(s-\gamma)]$; when $f\in\{f\in\sH_g:~g(\bw \cdot \bx -\gamma)+b\leq 0\}$, $f(\bx)=g(\bw\cdot \bx)+b\in [g(-\|\bx\|)-\C,\max_{-\|\bx\|\leq s\leq \|\bx\|}g(s)-g(s-\gamma)]$; when $f\in\{f\in\sH_g:~g(\bw \cdot \bx +\gamma)+b\geq 0\}$, $f(\bx)=g(\bw\cdot \bx)+b\in [\min_{-\|\bx\|\leq s\leq \|\bx\|}g(s)-g(s+\gamma),g(\|\bx\|)+\C]$. For convenience, we denote $\overline{A}(t)=\max_{-t\leq s\leq t}g(s)-g(s-\gamma)\geq 0$ and $\underline{A}(t)=\min_{-t\leq s\leq t}g(s)-g(s+\gamma)\leq 0$ for any $0\leq t \leq 1$. Therefore, for any $\bx\in \sX$, \eqref{eq:general_positive_proof1} is equivalent to 
\begin{equation}
\label{eq:general_positive_proof2}
	\begin{aligned}
    \inf\limits_{\underline{A}(\|\bx\|)\leq t\leq\overline{A}(\|\bx\|)}\bar{\cC}_{\phi}(t,\frac12)> &\inf\limits_{g(-\|\bx\|)-\C\leq t\leq g(\|\bx\|)+\C}\bar{\cC}_{\phi}(t,\frac12)\,,\text{and}\\
    \inf\limits_{g(-\|\bx\|)-\C\leq t\leq \overline{A}(\|\bx\|)}\bar{\cC}_{\phi}(t,\eta) > &\inf\limits_{g(-\|\bx\|)-\C\leq t\leq g(\|\bx\|)+\C}\bar{\cC}_{\phi}(t,\eta)  \text{ for all } \eta\in (\frac12,1]\,,\text{and}\\
    \inf\limits_{\underline{A}(\|\bx\|)\leq t\leq g(\|\bx\|)+\C}\bar{\cC}_{\phi}(t,\eta) > &\inf\limits_{g(-\|\bx\|)-\C\leq t\leq g(\|\bx\|)+\C}\bar{\cC}_{\phi}(t,\eta) \text{ for all } \eta\in [0,\frac12)\,.	
	\end{aligned}
\end{equation}
Suppose that $\phi$ is $\sH_g$-calibrated with respect to $\ell_{\gamma}$. Since for $\eta\in[0,\frac12)$,
\begin{align*}
    &\inf\limits_{\underline{A}(\|\bx\|)\leq t\leq g(\|\bx\|)+\C}\bar{\cC}_{\phi}(t,\eta)
    = \min\curl*{\bar{\cC}_{\phi}(\underline{A}(\|\bx\|),\eta),\bar{\cC}_{\phi}(g(\|\bx\|)+\C,\eta)} & \text{(Part~\ref{part3_lemma:quasiconcave_even} of Lemma~\ref{lemma:quasiconcave_even})}\\
    &\inf\limits_{g(-\|\bx\|)-\C\leq t\leq g(\|\bx\|)+\C}\bar{\cC}_{\phi}(t,\eta)
    = \min\curl*{\bar{\cC}_{\phi}(g(-\|\bx\|)-\C,\eta),\bar{\cC}_{\phi}(g(\|\bx\|)+\C,\eta)} & \text{(Part~\ref{part3_lemma:quasiconcave_even} of Lemma~\ref{lemma:quasiconcave_even})}
\end{align*}
we have $\bar{\cC}_{\phi}(g(-\|\bx\|)-\C,\eta)<\bar{\cC}_{\phi}(g(\|\bx\|)+\C,\eta)$ for any $\bx \in \sX$, otherwise
\begin{align*}
 \inf_{\underline{A}(\|\bx\|)\leq t\leq g(\|\bx\|)+\C}\bar{\cC}_{\phi}(t,\eta)\leq \bar{\cC}_{\phi}(g(\|\bx\|)+\C,\eta)=\inf_{g(-\|\bx\|)-\C\leq t\leq g(\|\bx\|)+\C}\bar{\cC}_{\phi}(t,\eta). 
\end{align*}
By \text{Part~\ref{part9_lemma:quasiconcave_even} of Lemma~\ref{lemma:quasiconcave_even}}, $\phi(\C-g(-t))+\phi(g(-t)-\C)=\phi(g(t)+\C)+\phi(-g(t)-\C)$ for all $0\leq t\leq1$.\\
Also, for any $0\leq t\leq 1$,
\begin{align*}
    &\frac12\min\curl*{\phi(\overline{A}(t))+\phi(-\overline{A}(t)),\phi(\underline{A}(t))+\phi(-\underline{A}(t)) }\\
    =& \inf\limits_{\underline{A}(t)\leq t\leq \overline{A}(t)}\bar{\cC}_{\phi}(t,\frac12) & \text{(Part~\ref{part3_lemma:quasiconcave_even} of Lemma~\ref{lemma:quasiconcave_even})}\\
    >& \inf\limits_{g(-t)-\C\leq t\leq g(t)+\C}\bar{\cC}_{\phi}(t,\frac12) & \eqref{eq:general_positive_proof2}\\
    =& \frac12\min\curl*{\phi(\C-g(-t))+\phi(g(-t)-\C),\phi(g(t)+\C)+\phi(-g(t)-\C) } & \text{(Part~\ref{part3_lemma:quasiconcave_even} of Lemma~\ref{lemma:quasiconcave_even})}\\
    =& \frac12(\phi(\C-g(-t))+\phi(g(-t)-\C))
\end{align*}

Now for the other direction, assume that for any $0\leq t \leq 1$,
  \begin{align*}
  \phi(\C-g(-t))+\phi(g(-t)-\C) & = \phi(g(t)+\C)+\phi(-g(t)-\C)\\
  \text{and} \quad
  \min\curl*{\phi(\overline{A}(t))+\phi(-\overline{A}(t)),\phi(\underline{A}(t))+\phi(-\underline{A}(t)) } & > \phi(\C-g(-t))+\phi(g(-t)-\C). 
 \end{align*}
Then for $\eta=\frac12$ and any $\bx \in \sX$,
\begin{align*}
    &\inf\limits_{\underline{A}(\|\bx\|)\leq t\leq\overline{A}(\|\bx\|)}\bar{\cC}_{\phi}(t,\frac12)\\
    =& \frac12\min\curl*{\phi(\overline{A}(\|\bx\|))+\phi(-\overline{A}(\|\bx\|)),\phi(\underline{A}(\|\bx\|))+\phi(-\underline{A}(\|\bx\|)) } & \text{(Part~\ref{part3_lemma:quasiconcave_even} of Lemma~\ref{lemma:quasiconcave_even})}\\
    >&\frac12(\phi(\C-g(-\|\bx\|))+\phi(g(-\|\bx\|)-\C)) & \text{(by assumption)}\\
    =& \frac12\min\curl*{\phi(\C-g(-\|\bx\|))+\phi(g(-\|\bx\|)-\C),\phi(g(\|\bx\|)+\C)+\phi(-g(\|\bx\|)-\C) } & \text{(by assumption)}\\
    =& \inf\limits_{g(-\|\bx\|)-\C\leq t\leq g(\|\bx\|)+\C}\bar{\cC}_{\phi}(t,\frac12). & \text{(Part~\ref{part3_lemma:quasiconcave_even} of Lemma~\ref{lemma:quasiconcave_even})}
\end{align*}
For $\eta\in(\frac12,1]$ and any $\bx \in \sX$,
\begin{align*}
    &\inf\limits_{g(-\|\bx\|)-\C\leq t\leq \overline{A}(\|\bx\|)}\bar{\cC}_{\phi}(t,\eta)
    = \min\curl*{\bar{\cC}_{\phi}(g(-\|\bx\|)-\C,\eta),\bar{\cC}_{\phi}(\overline{A}(\|\bx\|),\eta)} & \text{(Part~\ref{part3_lemma:quasiconcave_even} of Lemma~\ref{lemma:quasiconcave_even})}\\
    &\inf\limits_{g(-\|\bx\|)-\C\leq t\leq g(\|\bx\|)+\C}\bar{\cC}_{\phi}(t,\eta)  
    = \min\curl*{\bar{\cC}_{\phi}(g(-\|\bx\|)-\C,\eta),\bar{\cC}_{\phi}(g(\|\bx\|)+\C,\eta)} & \text{(Part~\ref{part3_lemma:quasiconcave_even} of Lemma~\ref{lemma:quasiconcave_even})}\\
    &= \bar{\cC}_{\phi}(g(\|\bx\|)+\C,\eta) & \text{(Part~\ref{part8_lemma:quasiconcave_even} of Lemma~\ref{lemma:quasiconcave_even})}
\end{align*}
Since $\phi$ is non-increasing, we have for any $\bx \in \sX$,
\begin{align*}
    &\phi(-g(\|\bx\|)-\C)-\phi(g(\|\bx\|)+\C)+\phi(\overline{A}(\|\bx\|))-\phi(-\overline{A}(\|\bx\|))\\
    \geq&  \phi(-g(\|\bx\|)-\C)-\phi(g(\|\bx\|)+\C)+\phi(g(\|\bx\|)+\C)-\phi(-g(\|\bx\|)-\C)\\
    =&0.
\end{align*}
Then for $\eta\in(\frac12,1]$ and any $\bx \in \sX$,
\begin{align*}
    &\bar{\cC}_{\phi}(\overline{A}(\|\bx\|),\eta)-\bar{\cC}_{\phi}(g(\|\bx\|)+\C,\eta)\\
    =& (\phi(\overline{A}(\|\bx\|))-\phi(-\overline{A}(\|\bx\|))+\phi(-g(\|\bx\|)-\C)-\phi(g(\|\bx\|)+\C))\eta+\phi(-\overline{A}(\|\bx\|))-\phi(-g(\|\bx\|)-\C)\\
    \geq& (\phi(\overline{A}(\|\bx\|))-\phi(-\overline{A}(\|\bx\|))+\phi(-g(\|\bx\|)-\C)-\phi(g(\|\bx\|)+\C))\frac12+\phi(-\overline{A}(\|\bx\|))-\phi(-g(\|\bx\|)-\C)\\
    =& \frac12(\phi(\overline{A}(\|\bx\|))-\phi(-\overline{A}(\|\bx\|))-\phi(-g(\|\bx\|)-\C)-\phi(g(\|\bx\|)+\C))\\
    >&0.
\end{align*}
In addition, by Part~\ref{part8_lemma:quasiconcave_even} of Lemma~\ref{lemma:quasiconcave_even}, for all $\eta\in (\frac12,1]$ and any $\bx \in \sX$, $\bar{\cC}_{\phi}(g(-\|\bx\|)-\C,\eta)-\bar{\cC}_{\phi}(g(\|\bx\|)+\C,\eta)>0$.
As a result, for $\eta\in (\frac12,1]$ and any $\bx \in \sX$,
\begin{align*}
    &\inf\limits_{g(-\|\bx\|)-\C\leq t\leq \overline{A}(\|\bx\|)}\bar{\cC}_{\phi}(t,\eta) -\inf\limits_{g(-\|\bx\|)-\C\leq t\leq g(\|\bx\|)+\C}\bar{\cC}_{\phi}(t,\eta)\\
    =& \min\curl*{\bar{\cC}_{\phi}(g(-\|\bx\|)-\C,\eta)-\bar{\cC}_{\phi}(g(\|\bx\|)+\C,\eta), \bar{\cC}_{\phi}(\overline{A}(\|\bx\|),\eta)-\bar{\cC}_{\phi}(g(\|\bx\|)+\C,\eta)}\\
    >&0.
\end{align*}
Finally, for $\eta\in[0,\frac12)$, by \text{Part~\ref{part9_lemma:quasiconcave_even} of Lemma~\ref{lemma:quasiconcave_even}}, we have $\bar{\cC}_{\phi}(g(-\|\bx\|)-\C,\eta)<\bar{\cC}_{\phi}(g(\|\bx\|)+\C,\eta)$ and 
\begin{align*}
    &\inf\limits_{\underline{A}(\|\bx\|)\leq t\leq g(\|\bx\|)+\C}\bar{\cC}_{\phi}(t,\eta)
    = \min\curl*{\bar{\cC}_{\phi}(\underline{A}(\|\bx\|),\eta),\bar{\cC}_{\phi}(g(\|\bx\|)+\C,\eta)} & \text{(Part~\ref{part3_lemma:quasiconcave_even} of Lemma~\ref{lemma:quasiconcave_even})}\\
    &\inf\limits_{g(-\|\bx\|)-\C\leq t\leq g(\|\bx\|)+\C}\bar{\cC}_{\phi}(t,\eta)  
    = \min\curl*{\bar{\cC}_{\phi}(g(-\|\bx\|)-\C,\eta),\bar{\cC}_{\phi}(g(\|\bx\|)+\C,\eta)} & \text{(Part~\ref{part3_lemma:quasiconcave_even} of Lemma~\ref{lemma:quasiconcave_even})}\\
    &= \bar{\cC}_{\phi}(g(-\|\bx\|)-\C,\eta) & \text{(Part~\ref{part9_lemma:quasiconcave_even} of Lemma~\ref{lemma:quasiconcave_even})}
\end{align*}
Since $\phi(\underline{A}(\|\bx\|))+\phi(-\underline{A}(\|\bx\|))>\phi(\C-g(-\|\bx\|))+\phi(g(-\|\bx\|)-\C)$ and $\phi$ is non-increasing, we have for any $\bx \in \sX$,
\begin{align*}
    &\phi(\C-g(-\|\bx\|))-\phi(g(-\|\bx\|)-\C)+\phi(\underline{A}(\|\bx\|))-\phi(-\underline{A}(\|\bx\|))\\
    =&  \phi(\C-g(-\|\bx\|))-\phi(-\underline{A}(\|\bx\|))+\phi(\underline{A}(\|\bx\|))-\phi(g(-\|\bx\|)-\C)\\
    <&  \phi(\underline{A}(\|\bx\|))-\phi(g(-\|\bx\|)-\C)+\phi(\underline{A}(\|\bx\|))-\phi(g(-\|\bx\|)-\C)\\
    =&  2[\phi(\underline{A}(\|\bx\|))-\phi(g(-\|\bx\|)-\C)]\\
    \leq&0.
\end{align*}
Then for $\eta\in[0,\frac12)$ and any $\bx \in \sX$.
\begin{align*}
    &\bar{\cC}_{\phi}(\underline{A}(\|\bx\|),\eta)- \bar{\cC}_{\phi}(g(-\|\bx\|)-\C,\eta)\\
    =& [\phi(\C-g(-\|\bx\|))-\phi(g(-\|\bx\|)-\C)+\phi(\underline{A}(\|\bx\|))-\phi(-\underline{A}(\|\bx\|))]\eta+\phi(-\underline{A}(\|\bx\|))-\phi(\C-g(-\|\bx\|))\\
    \geq& [\phi(\C-g(-\|\bx\|))-\phi(g(-\|\bx\|)-\C)+\phi(\underline{A}(\|\bx\|))-\phi(-\underline{A}(\|\bx\|))]\frac12+\phi(-\underline{A}(\|\bx\|))-\phi(\C-g(-\|\bx\|))\\
   =& \frac12[\phi(\underline{A}(\|\bx\|))+\phi(-\underline{A}(\|\bx\|))-\phi(g(-\|\bx\|)-\C)-\phi(\C-g(-\|\bx\|))]\\
    >&0.
\end{align*}
In addition, by Part~\ref{part9_lemma:quasiconcave_even} of Lemma~\ref{lemma:quasiconcave_even}, for all $\eta\in[0,\frac12)$ and any $\bx \in \sX$, $\bar{\cC}_{\phi}(g(\|\bx\|)+\C,\eta)-\bar{\cC}_{\phi}(g(-\|\bx\|)-\C,\eta)>0$. As a result, for $\eta\in[0,\frac12)$ and any $\bx \in \sX$,
\begin{align*}
    &\inf\limits_{\underline{A}(\|\bx\|)\leq t\leq g(\|\bx\|)+\C}\bar{\cC}_{\phi}(t,\eta) -\inf\limits_{g(-\|\bx\|)-\C\leq t\leq g(\|\bx\|)+\C}\bar{\cC}_{\phi}(t,\eta)\\
    =& \min\curl*{\bar{\cC}_{\phi}(g(\|\bx\|)+\C,\eta)-\bar{\cC}_{\phi}(g(-\|\bx\|)-\C,\eta), \bar{\cC}_{\phi}(\underline{A}(\|\bx\|),\eta)-\bar{\cC}_{\phi}(g(-\|\bx\|)-\C,\eta)}\\
    >&0.
\end{align*}
\end{proof}

\end{document}